\def\year{2022}\relax
\documentclass[letterpaper]{article} 
\usepackage{aaai22}  
\usepackage{times}  
\usepackage{helvet}  
\usepackage{courier}  
\usepackage[hyphens]{url}  
\usepackage{graphicx} 
\usepackage{natbib}  
\usepackage{caption} 
\DeclareCaptionStyle{ruled}{labelfont=normalfont,labelsep=colon,strut=off} 
\usepackage{algorithm}
\usepackage{algorithmic}

\usepackage{newfloat}
\usepackage{listings}
\floatstyle{ruled}
\newfloat{listing}{tb}{lst}{}
\floatname{listing}{Listing}
\usepackage{mathrsfs}
\usepackage{amsthm}
\usepackage{amsmath,bm,bbm}
\usepackage{amssymb}
\usepackage{multirow}
\usepackage{tabularx}
\usepackage{multicol}
\usepackage{subfigure}
\usepackage{url}
\usepackage{mathtools}
\usepackage{titletoc}

\usepackage{float}
\usepackage{threeparttable}

\usepackage{thmtools,thm-restate}

\newtheorem{mythm}{Theorem}

\newtheorem{myprops}{Proposition}

\newtheorem{myasum}{Assumption}

\usepackage{algorithm}
\usepackage{algorithmic}

\DeclareMathOperator*{\argmax}{\arg\,\max}

\newcommand{\fig}{{Figure }}
\newcommand{\tables}{{Table }}
\newcommand{\eq}{{Eq. }}
\newcommand{\alg}{{Algorithm }}

\newcommand{\ourmethod}{\text{MaxIn-Elo}}
\newcommand{\ourmethodii}{\text{MaxIn-mElo}}

\newcommand{\yali}[1]{{ \color{magenta}[yali: #1]}}
\newcommand{ \du}[1]{\textcolor{red}{#1}}

\newcommand{\xu}[1]{\textcolor{red}{[xu: #1]}}

\def\rc{\color{magenta}}

\usepackage{changes}
\usepackage{hyperref}

\usepackage{balance}
\usepackage{tikz-cd}
\usepackage{wrapfig,lipsum,booktabs}
\usepackage{adjustbox}

\title{Learning to Identify Top Elo Ratings:\\
A Dueling Bandits Approach}

\author{
 Xue Yan\textsuperscript{\rm 1}\textsuperscript{\rm 2}, 
  Yali Du \thanks{Corresponding to Yali Du $\langle$\href{mailto:yali.du@kcl.ac.uk}{yali.du@kcl.ac.uk}$\rangle$. } \textsuperscript{\rm 3},
  Binxin Ru \textsuperscript{\rm 4},
  Jun Wang \textsuperscript{\rm 5},
  Haifeng Zhang\textsuperscript{\rm 1}\textsuperscript{\rm 2},
  Xu Chen\textsuperscript{\rm 6}
 }

\affiliations{
    \textsuperscript{\rm 1} Institute of Automation, Chinese Academy of Sciences, China\\
    \textsuperscript{\rm 2} School of Artificial Intelligence, University of Chinese Academy of Sciences, China\\
    \textsuperscript{\rm 3} King's College London, UK
    \textsuperscript{\rm 4} University of Oxford, UK
    \textsuperscript{\rm 5} University College London, UK\\
    \textsuperscript{\rm 6} Beijing Key Laboratory of Big Data Management and Analysis Methods, GSAI, Renmin University of China.

}

\begin{document}

\maketitle

\begin{abstract}


The Elo rating system is widely adopted to evaluate the skills of (chess) game and sports players. Recently it has been also integrated into  machine learning algorithms in evaluating the performance of computerised AI agents. However, an accurate estimation of the Elo rating (for the top players) often requires many rounds of competitions, which can be expensive to carry out. 
In this paper, 
to improve the sample efficiency of the Elo evaluation (for top players), we propose an efficient online match scheduling algorithm. Specifically, we identify and match the top players through a dueling bandits framework and tailor the bandit algorithm to the gradient-based update of Elo. We show that it reduces the per-step memory and time complexity to constant, compared to the traditional likelihood maximization approaches requiring $O(t)$ time.
Our algorithm has a regret guarantee of $\tilde{O}(\sqrt{T}$), sublinear in the number of competition rounds and has been extended to the multidimensional Elo ratings for handling intransitive games. We empirically demonstrate that our method achieves superior convergence speed and time efficiency on a variety of gaming tasks.

\end{abstract}

\section{Introduction}

In this paper, we investigate the selection of best multi-agent strategies under the Elo rating systems.
The evaluation of the competition outcome has received lots of attention, especially in view of the successful usage of reinforcement learning in StarCraft \cite{Alphastar, han2019grid,du2019liir}, Game of Go \cite{silver2017mastering} and video games \cite{mnih2015human}.
The Elo rating system \citep{elo1978rating} is a predominant and valuable algorithm for evaluating and ranking agents. In the widely adopted Bradley-Terry model \citep{hunter2004mm} for Elo, each player is assigned a numerical rating which is updated with competition outcomes via online stochastic gradient descent. 
Further, for dealing with non-transitive relations between interacting agents such as the game of \textit{Rock-Paper-Scissors}, \citet{balduzzi2018re} proposes multidimensional Elo (mElo), which decomposes a game into transitive and cyclic parts to handle intransitive skills and evaluates different strategies by computing Nash-averaging. 

In practical settings when a competition is expensive to conduct, updating Elo rating in a sample efficient way is highly valuable. To achieve such sample efficiency, we need a way to select the most informative pairs for evaluation. 
Two popular sampling approaches are Round-robin \citep{rasmussen2008round} and Elimination tournament \citep{groh2012optimal}; The Round-robin \citep{rasmussen2008round} is widely used in sport scheduling to balance the total time, venue usage and fairness of tournaments. It would arrange each team to play against all the others in as few as possible days while satisfying some constraints such as each team 
not playing twice in the same day to promote game fairness. By contrast, the Elimination tournament \citep{groh2012optimal} only allows the winners at each round to proceed to the next round, so the stronger team will have the chance to play more times. A recent approach, RG-UCB, \citet{rowland2019multiagent} introduces an adaptive sampling scheme to estimate the accurate ranking among all agents. RG-UCB considers sampling of agent match-ups as a collection of pure exploration bandit problems \citep{bubeck2011pure} and requires enough pairwise comparison for estimating each pair of strategies.  

However, these tournament matching/sampling methods suffer from two major limitations which prohibit their wide usage in the modern large scale evaluations.
Firstly, both the Round-robin and the Elimination tournament organise competitions following a pre-designed schedule, and the Elimination tournament scheduling may need some prior knowledge on the players' skill. Also each pair of players only compete once in both schemes so the results can be highly noisy.
Secondly, the main idea behind the matching schemes of the RG-UCB and the Round-robin is random sampling, which fails to pay more attention to more promising players and/or pairs with higher uncertainty in competition outcome. Thus, they are less sample efficient in identifying the best players.

In this work, we propose two sampling algorithms, named MaxIn-Elo and MaxIn-mElo, for the update of Elo and mElo rating systems respectively. Specifically, we maintain a candidate set with promising players using UCB-based (Upper Confidence Bound) dueling bandits and {then} select the pair with {the} highest uncertainty in competition outcome at each round.
{On the one hand}, our algorithms are adapted to the gradient-based update of Elo rating systems, thus more memory and time efficient compared to a prior work \citet{saha2020regret} which relies on maximum likelihood estimation (MLE). {
One the other hand, we extend our method to update mElo (multidimensional Elo) ratings to handle intransitive games, while \citet{saha2020regret} is based on a generalized linear model and can only fit to the transitive games.}
{To the best of our knowledge}, this is the first work that {enables online gradient-based update for} dueling bandits{, and a theoretical guarantee on the cumulative regret is provided.}
Also compared to a previous dueling bandit method descending through randomly sampled gradients at each time step \citep{yue2009interactively}, our method, by selecting the pairs with higher information gains
from a set of top candidates, is more sample efficient and are guaranteed to converge at $\tilde{O}(\sqrt{T})$\footnote{$\tilde{O}$ ignores poly-logarithmic factors}.

In summary, our contributions are three-fold:
Firstly, we are the first to propose two online active sampling algorithms $\ourmethod$ and $\ourmethodii$ that select maximum informative pairs with dueling bandits to update Elo and mElo ratings. 
Secondly, we give the regret analysis of our proposed algorithms and show that the algorithms converge at $\tilde{O}(\sqrt{T})$.
Thirdly, we demonstrate empirically on synthetic and real-world games that our algorithms achieve significantly lower cumulative regret than all baselines. Notably, our methods outperform MaxInP \cite{saha2020adversarial} which uses maximum likelihood for more accurate estimation while with lower time and memory complexity.
{Code of this project is available at \url{https://github.com/yanxue7/MaxIn-Elo.git}.}
\section{Related Work}

Multi-agent evaluation has attracted wide attention in ranking of players
\cite{silver2017mastering,lai2015giraffe,arneson2010monte,gruslys2018reactor}
and in selecting stronger strategies in meta games \cite{muller2020generalized,czarnecki2020real}. 
{There are many methods used for multi-agent evaluation problem.} 
The Elo rating system is widely used for two-player games such as chess and tennis. It increases (decreases) player's rating according to player wins (loss) a competition, and updates ratings by online stochastic gradient descent (SGD),  which is computationally efficient and simple to implement. While the Elo rating system cannot handle intransitive games such as rock-paper-scissors, multidimensional Elo (Melo) \citep{balduzzi2018re} was introduced. It decomposes the win-loss matrix of an intransitive game into the transitive component and cyclic component baked in Hodge decomposition theory \cite{jiang2011statistical}. 
{
 $\alpha$-rank \cite{omidshafiei2019alpha} is another popular counterpart in tackling intransitive games; recent attempts  improve its sample efficiency based on noisy comparisons \cite{du2021estimating,rowland2019multiagent,omidshafiei2019alpha} and scalability by stochastic optimization \cite{yang2020alphaalpha}. }
Despite various evaluation algorithms discussed, how to sample agent pairs at each round is of high value to realize these algorithms in large-scale evaluation tasks. 

{ We consider dueling bandits for online match scheduling in the evaluation of players.
}
The concept of dueling bandits was firstly proposed in \cite{yue2009interactively}.
Compared to traditional bandits algorithms which pull one arm at each round and receive the reward of this arm directly, dueling bandits pull arm-pair at each round and only get a binary comparison result. 
DBGD  \citep{yue2009interactively} models a convex optimization problem as the dueling bandits problem which aims to find the best point in a convex space, and DBGD uses a random gradient as the direction of exploration for selecting the next arm-pair.  
\citet{yue2012k} formulates the best player identification as a dueling bandits problem with noisy comparison results and an underlying winning probability matrix, and proposes two algorithms as well as their corresponding regret bounds.
{
These algorithms identify best arms based on the observed binary feedback however do not  learn the player's skills (ratings), which is helpful in predicting future competition outcomes.} \citet{szorenyi2015online} regards the ranking of $M$ alternatives (e.g. human players or agents) as a dueling bandits problem. They introduce the confidence interval of rating or winning probability into the dueling bandits problem, and design algorithms to identify the close-to-optimal item or to obtain the close-to-optimal whole ranking respectively.
\citet{saha2020adversarial} studies the adversarial setting, in which the winning probability is non-stationary because players' skill may change over time. And they measure arms' abilities by estimating Borda score, however Borda score does not possess predictive power of future competition results and the algorithm of \cite{saha2020adversarial} estimates Borda score only by simply calculating the frequency of wins. \citet{heckel2019active} gives an active ranking algorithm that can solve the top-$k$ player identification problem and find the entire sequential ranking among all players.

While existing algorithms could rely on Borda score to update the rankings or design specific sort algorithms to obtain the ranking of items, they are not suitable for the Elo rating systems that adopt stochastic gradient descent to update ratings.
Eearlier attempt \citep{ding2021efficient} proposes SGD-TS for contextual bandit problem, which learns parameters in the generalized linear model through online SGD instead and employs Thompson Sampling (TS) \citep{thompson1933likelihood,agrawal2012analysis,agrawal2013thompson} to encourage exploration in arm-pulling. Compared to UCB-GLM \citep{li2017provably} that adopt maximum likelihood estimators, SGD-TS achieves a similar theoretical cumulative regret bound, but lower time and memory complexity.
However, no prior work has studied the SGD update in dueling bandits setting.

In this work, we will tame dueling bandits for the online match scheduling in the Elo rating system that adopts  stochastic gradient descent.
\citet{saha2020regret} proposes the algorithm that selects Maximum-Informative-Pair (MaxInP) for $K$-armed contextual dueling bandits. This algorithm utilizes the MLE 
to estimate parameter $\hat{\theta}$ and uses UCB \citep{auer2002finite} estimator to narrow down the set of candidate pairs from which the pair of arms with the maximum uncertainty is selected at each round. Our algorithms adopt a similar design as MaxInP \citet{saha2020regret} in calculating the uncertainty of a pair. However, we use an online batch SGD instead of MLE to update Elo rating, which is more time and memory efficient.





\section{Methodology}
\subsection{Background}

Suppose there are $n$ players, the Elo rating system \citep{elo1978rating} assigns a rating $r_x, x\in [n]$ to each player representing its skill. 
Let $r^*$ denote the true ratings of $n$ players. Our aim is to identify the best player among all $n$ players:
\begin{equation}
    x^*=\arg \max_{x\in[n]} r^*_x
\end{equation}
Denote $ P$ as the true winning probability matrix, $p_{xy}$ as the underlying groundtruth probability of $x$ beating $y$. Based on the Bradley-Terry model \cite{hunter2004mm}, the predicted probablility of player $x$ winning $y$ is 
\begin{equation}
    \label{"eq:elo_hat"}
    \hat{p}_{xy}=\sigma ( r_x- r_y).
\end{equation} 
$\sigma(x)$ is a sigmoid function with $\sigma(x)=\frac{1}{1+e^{-x}}$. 
Elo ratings are updated by maximizing the likelihood of win-loss predictions which corresponds to minimizing the loss: 
\begin{equation}\label{eq:elo-objective}
\ell_{\text {Elo }}\left(p_{xy}, \hat{p}_{xy}\right)=-p_{xy} \log \hat{p}_{xy}-\left(1-p_{xy}\right) \log \left(1-\hat{p}_{xy}\right).
\end{equation}
At time $t$ player $x$ compete with player $y$ with outcome $o_{xy}^t$: $o_{xy}^t=1$ if $x$ wins and $o_{xy}^t=0$ otherwise. We can use $o_{xy}^t$ to compute the gradient of Eq. \eqref{eq:elo-objective} and update Elo by gradient descent:
\begin{equation}\label{eq:elo-grad-descent}
r_{x}^{t+1} \leftarrow r_{x}^{t}-\eta \cdot \nabla_{r_{x}} \ell_{\text {Elo }}\left(o_{xy}^{t}, \hat{p}_{xy}^{t}\right)=r_{x}^{t}+\eta \cdot\left(o_{xy}^{t}-\hat{p}_{xy}^{t}\right).
\end{equation}


Let $T$ denote the total number of rounds, at each round $t \in[T]$, we adopt a system that will pull a pair of players $(x_t,y_t)\in[n]\times[n]$ and get comparison result $o_t(x_t,y_t)\sim \text{Bern}(p_{xy}).$ 
The cumulative regret of $T$ rounds is defined as 
\begin{equation}\label{"eq:cumu reg"}
    R(T) = \sum_{t=1}^T [ r^*_{x^*} - \frac{1}{2}( r^*_{x_t}+ r^*_{y_t})].
\end{equation}
The definition is consistent with \citep{saha2020regret}. It measures the reward difference between the best arm and the two selected arms at each round.

\subsection{MaxIn-Elo Algorithm}

We use the notations below in the followed presentations.

$\bullet$ $n$: the number of players.

$\bullet$ $\tau$: the batch size.

$\bullet$ $r$: a vector of $n$ players' ratings. 

$\bullet$ $r^*$: the true ratings of $n$ players.

$\bullet$ $\hat{r}_t$: { estimator by MLE with $t$ round comparisons.}

$\bullet$ $\tilde{r}_j$: the SGD estimator at batch $j$.

$\bullet$ $\bar{r}=\sum_{q=1}^j\tilde{r}_q$: the average of previous SGD iterations. 

$\bullet$ $\| x \|=\sqrt{x^\top x}$: the standard $\ell_2$ norm.

$\bullet$ $e_i$: the $i$-th unit base vector, {i.e., the $i$-dimension equals $1$ and all other components equal $0$.}

$\bullet$ $V_t$: the history matrix recording previous $t$ pulling information defined by $V_t=\sum_{i=1}^{t-1}(e_{x_i}-e_{y_i})(e_{x_i}-e_{y_i})^\top.$

$\bullet$ $\| x\|_V$: a special $\ell_2$-norm associated with matrix $V$ defined by $\| x\|_V=\sqrt{x^\top V x}$.


$\bullet$ $\mathcal{B}$ a neighborhood of  ${r}^*$ with $\mathcal{B}=\{r|\left\|r-{r}^*\right\|\leq 3\}$.

$\bullet$ $\mathcal{C}$: a  neighborhood of 
$\hat{r}_\tau$ with $\mathcal{C}=\{r|\left\|r-\hat{r}_\tau\right\|\leq 2\}$.

$\bullet$ $\prod_\mathcal{C}(.)$: the projection operation defined by:
\begin{align}\label{eq:projection}
\prod_\mathcal{C}(r)=\hat{r}_\tau+\frac{2*(r-\hat{r}_\tau)}{\min\{2,\|r-\hat{r}_\tau\|\}}
\end{align}


\begin{algorithm}[t!]
\caption{$\ourmethod$: Dueling bandits with online SGD for top player identification.  }
\label{alg:dueling-sgd-UCB}
\begin{algorithmic}[1]
    \REQUIRE batch size $\tau$, maximum number of rounds $T$, $N$ players' strategies, parameters $\alpha,\gamma$.
    \ENSURE output $ r$
    \STATE Randomly choose a pair to compare and record as $x_t,y_t, o_t$ for $t \in [\tau]$
    \STATE $V_{\tau+1}=\sum_{t=1}^\tau (e_{x_t}-e_{y_t})(e_{x_t}-e_{y_t})^T$
    \STATE Calculate the maximum-likelihood estimator $\hat{ r}_\tau $ by solving \\
    $\nabla_{ r}\sum_{t=1}^\tau \ell_{\text {Elo}} \left(o_t,\hat{p}(x_t,y_t)\right)=0 $
    \STATE Maintain convex set $\mathcal{C}=\left\{r:\left\|r-\hat{r}_{\tau}\right\| \leq 2\right\}$
    
     \FOR{ $t=\tau+1, \tau+2, \dots, T$}
        \IF{ $t\% \tau =1$ }
        \STATE $j \leftarrow \lfloor(t-1)/\tau \rfloor $ and $\eta_j = \frac{1}{\alpha j}$
        \STATE 
        Calculate gradient $\nabla_{r}l_{j,\tau}(\tilde{r}_{j-1})$ through \eq \eqref{eq:elo-grad-descent}
        \STATE  Update ratings $\tilde{r}_j$ through \eq \eqref{eq:update_jth_ratings}
        \STATE Compute $\bar{r}=\frac{1}{j} \sum_{q=1}^{j} \tilde{r}_{q}$
        \ENDIF
        \STATE Define a candidate optimal set $\mathcal{S}=\{x\ |\ \bar{r}_x-\bar{r}_y+\gamma\
\left\|e_x-e_y\right\|_{V_{t}^{-1}}>0 ,\ \forall y \in[n]/\{x\}\}$
        \STATE Select a pair as:\\
        \quad $(x_t,y_t)=\argmax_{(x,y) \in \mathcal{S}}\left\|e_x-e_y\right\|_{V_{t}^{-1}}$.
        \STATE Let players $(x_t, y_t)$ compete and observe $o_t(x_t, y_t)$
        \STATE Compute $V_{t+1}=V_{t}+(e_{x_t}-e_{y_t})(e_{x_t}-e_{y_t})^T$
        \ENDFOR
\end{algorithmic}
\end{algorithm}

\paragraph{Algorithm Overview}
The main idea of our MaxIn-Elo algorithm is to maintain a candidate set of promising items via UCB and select the most informative pairs out of the set to evaluate at each round. 
Firstly, the ratings are initialized by maximizing likelihood of \eq \eqref{eq:elo-objective} on a batch of randomly sampled pairs with batch size $\tau$. The solution is denoted as $\hat{r}_\tau$ and $\tilde{r}_0=\hat{r}$. Then starting from round $t=\tau+1$, we update $\tilde{r}_j$ every $\tau$ rounds by solving the following objective function
\begin{align}\label{eq:ljtau_mle_obj}
        l_{j,\tau}( r) =\sum_{t=(j-1)\tau+1}^{j\tau} \ell_{\text {Elo }}(o_t,\hat{p}(x_t,y_t)).
\end{align}
The stochastic gradient update of $\tilde{r}_j$ reads
 \begin{align}\label{eq:update_jth_ratings}
          \tilde{r}_{j} \leftarrow \prod_{\mathcal{C}}\left(\tilde{r}_{j-1}-\eta_{j} \nabla_{r}l_{j,\tau}(\tilde{r}_{j-1}) \right).
\end{align}
First,
the strong convexity of the objective function is required for fast convergence, 
and if we select a suitable $\tau$ through \eq \eqref{"eq:tau"}, the aggregated objective function $l_{j,\tau}( r)$ is a $\alpha$-strong convex function when $r\in\mathcal{B}$. 
Second, to ensure $\tilde{r}_j\in\mathcal{B}$, $\tilde{r}_j$ is projected into the convex set $\mathcal{C}$ (also discussed in the proof of Lemma \ref{"lm:mylemma2"}).

For each update, a batch of pairs are selected that lead to maximal information gain.  The UCB score of a pair is defined by:
\begin{equation}
\label{"eq:UCB_esti"}
    h(x_t,y_t)=\bar{r}_{x_t}-\bar{r}_{y_t}+\gamma\left\| e_{x_t}-e_{y_t}\right\|_{V_{t}^{-1}},
\end{equation} 
with the balance parameter $\gamma$. The specific $V_t^{-1}$ norm $\gamma\left\| e_{x_t}-e_{y_t}\right\|_{V_{t}^{-1}}$ measures the uncertainty between two arms. The UCB estimator balances the exploitation and exploration through combining ratings estimation $\bar{r}$ and the uncertainty term.  

At each round $t$, we obtain a set of optimal player candidates $\mathcal{S}$ with positive UCB scores:
\begin{equation}
    \label{"eq:select_candidate"}
    \mathcal{S}=\{x|h(x,y)>0,\ \forall\ y \in[n]/\{x\}\}.
\end{equation}
From the candidate set $\mathcal{S}$, we then pull a pair of arms with highest uncertainty by 
\begin{align}\label{'eq:select_pair_largest_uncer'}
(x_t,y_t)=\argmax_{(x,y)\ \in\mathcal{S}\times\mathcal{S}}\left\|e_x-e_y\right\|_{V_{t}^{-1}}
\end{align} 
to induce sufficient exploration. A detailed algorithm of our $\ourmethod$ is shown in \alg \ref{alg:dueling-sgd-UCB}.
 


 
Compared to \cite{saha2020regret} which uses MLE at each iteration, $\ourmethod$ uses SGD to update the Elo rating $r$ as traditional Elo does. Thus, our method is more efficient in both computation and time, and simple to implement. {Because the objective function $l_{j,\tau}$ is the aggregation of a batch with $\tau$ steps for the $\alpha$-strong convexity, the online SGD on this objective function is actually the mini-batch gradient descent on $\tau$ single steps.} The mini-batch update provides diversity and makes our algorithm more smooth and steady. {Thus, the selection of the batch size $\tau$ is important to achieve the fast convergence speed and smoothness during learning.} 
While RG-UCB randomly selects a pair to evaluate, our $\ourmethod$ selects the maximum informative pair and trade-off exploration and exploitation. 
To our best knowledge, this is the first algorithm that allows stochastic gradient descent update in dueling bandits settings. 
{See Table\ref{tab:complexity} for a comparison on time and memory. }

\subsection{MaxIn-mElo Algorithm}

To enable the rating system to handle the intransitive skills, we extend the  online sampling algorithm to multidimensional Elo ratings (mElo) \cite{balduzzi2018re}. 
Baking in the Hodge decomposition theory \citep{jiang2011statistical}, 
mElo proposed to decompose the antisymmetric logits matrix of win-loss probabilities into a transitive component, i.e. gradient flow of rating vector, and a cyclic component to capture the intransitive relations. 
By learning a $2k$-dimensional vector $c_x$ and a rating $r_x$ per player, the win-loss prediction for mElo$_{2k}$ is defined as:
\begin{equation}\label{"eq:melo_qhat"}
\hat{p}_{xy}=\sigma\left(r_{x}-r_{y}+{c}_{x}^{\top} \cdot {\Omega}_{2 k \times 2 k} \cdot {c}_{y}\right).
\end{equation}
where ${\Omega}_{2 k \times 2 k}=\sum_{i=1}^{k}({e}_{2 i-1} {e}_{2 i}^{\top}-{e}_{2 i} {e}_{2 i-1}^{\top})$.

The UCB estimate of a pair $(x_t,y_t)$ for mElo then becomes:
\begin{equation}\label{"eq:UCB_esti-melo"}
    h(x_t,y_t)=\bar{r}_{x_t}-\bar{r}_{y_t}+\bar{c}_x^T \Omega \bar{c}_y + \gamma\left\| e_{x_t}-e_{y_t}\right\|_{V_{t}^{-1}}.
\end{equation} 
Notice that compared to Elo ratings with $k=0$ (\eq \eqref{"eq:elo_hat"}), mElo ratings assign a feature vector per player to  approximated intransitive interactions.
We present the details for the mElo ratings and Algorithm 2 for $\ourmethodii$ in Appendix. 


\begin{table}[t!]
\centering
\caption{
Comparison of regret, time complexity and memory  with other algorithms. 
Our MaxIn-Elo and the MaxInP achieve the lowest regret bound $\tilde{O}(\sqrt{T})$, but our MaxIn-Elo has lower time and memory complexity than the MaxInP.
}
\begin{tabular}{cccc}
\hline Algorithms & Regret & Time Complexity & Memory\\
\hline 
DBGD & $O(T^{2/3})$ & $O(T)$ & $O(n)$ \\
RG-UCB & No & $O(T)$ & $O(n)$  \\
Random & No & $O(T)$ & $O(n)$ \\
MaxInP & $\tilde{O}(\sqrt{T})$ & $O(nT^2+n^2T)$ & $O(nT)$ \\
\textbf{MaxIn-Elo} & $\tilde{O}(\sqrt{T})$ & $O(n^2T)$ & $O(n^2)$ \\
\hline
\end{tabular}
\label{tab:complexity}
\end{table}

\section{Regret analysis}\label{sec:regret analysis}

We give the cumulative regret bound of $\ourmethod$, as for as we know, this is the first {work} that combines the online {gradient update} with dueling bandits and gives the  cumulative regret of dueling bandits while being  updated with SGD.  
We make a mild assumption on the link function $\sigma$.


\begin{myasum}
\label{"ass1:"}
Define $c_\eta=\inf_{\left\{\left\| r-r^{*}\right\| \leq \eta \right\}} \sigma^{\prime}\left(r_x-r_y\right)$, where $(x,y)\in [n]\times[n]$, and we assume $c_3>0$.
\end{myasum}
This assumption is similar to that in  \citep{ding2021efficient}. Our main results rely on the following concentration events and the proofs of which are deferred to Appendix.
\begin{restatable}{lem}{lemmai}
\label{"lm:mylemma1"}
Suppose we sample a sequence of arm pairs $\{(x_1,y_1), (x_2,y_2),\dots,(x_t,y_t)\}$ through \alg \ref{alg:dueling-sgd-UCB} up to round t, and assume the selected batch size $\tau$  satisfy that $\lambda_{\min }(V_{\tau+1}) \geq 1$, where $\lambda_{\min}(V_{\tau+1})$ means the minimum eigenvalue of the matrix $(V_{\tau+1})$, Then $\forall t>0$, 
$$
\sum_{i=\tau+1}^{\tau+t}\left\|\left(e_{x_i}-e_{y_i}\right)\right\|_{V_{i}^{-1}} < \sqrt{2 n t \log \left(\frac{2\tau+t}{n}\right)}.
$$
\end{restatable}
{{Lemma} \ref{"lm:mylemma1"} gives the bound of the sum of selected pair's uncertainty from round $\tau+1$ to $t$.}
And this lemma will be adopted to derive the cumulative regret bound. 
In the following Lemma \ref{"lm:mylemma2"}, we show that when the batch size $\tau$ is chosen as
\eq \eqref{"eq:tau"},  we have the concentration property of the {averaged} SGD estimator $\bar{r}$.
\begin{restatable}{lem}{lemmaii}
\label{"lm:mylemma2"}
{
Assume that there exists a positive constant $\lambda_f$ such that $\lambda_{\min}\left(\mathbb{E}[(e_{x_t}-e_{y_t})(e_{x_t}-e_{y_t})^T]\right)\geq \lambda_f$ holds at each round $t>\tau$, where $(x_t,y_t)$ is sampled through \alg \ref{alg:dueling-sgd-UCB}. } 
Let the batch size $\tau$ satisfies
\begin{equation}\label{"eq:tau"}
\begin{aligned}
\tau_{1} &=2\left(\frac{C_{1} \sqrt{n}+C_{2} \sqrt{2 \log T}}{\lambda_{\min}(B)}\right)^2+\frac{16(n+2 \log T)}{c_{1}^{2} \lambda_{\min}(B)}, \\
\tau_{2}& =2\left(\frac{C_{1} \sqrt{n}+C_{2} \sqrt{2 \log T}}{\lambda_f}\right)^{2}+\frac{4 \alpha}{c_{3} \lambda_f}, \\
\tau &=\left\lceil\max \left\{\tau_{1}, \tau_{2}\right\}\right\rceil,
\end{aligned}
\end{equation}
where $B=\mathbb{E}_{(x,y)\overset{iid}{\sim}[n]\times[n]}\left[(e_x-e_y)(e_x-e_y)^T\right]$.
Define $g_{1}(t)$ and $g_{2}(j)$, 
\begin{align}\label{eq:g1g2}
    g_{1}(t) &=\frac{1}{2c_{1}} \sqrt{\frac{n}{2} \log \left(1+\frac{2 t}{n}\right)+2 \log T}, \\
    g_{2}(j) &=\frac{\tau}{\alpha} \sqrt{1+\log j}.
\end{align}
For a constant $\alpha \geq c_3$, there exists two positive constants $C_1$, $C_2$ such that if the batch size $\tau$ is chosen as \eq\eqref{"eq:tau"},
then we have that at each round $t>\tau$ corresponding to batch $j=\lfloor\frac{t-1}{\tau}\rfloor$, event $E_1(t)$ holds with probability at least $1-\frac{5}{T^2}$, where $E_{1}(t)=\{\forall (x,y):\left|( e_x- e_y)^T\left(\bar{r}_{j}-r^{*}\right)\right| \leq g_{1}(j\tau)\|e_x-e_y\|_{V_{j\tau+1}^{-1}}+g_{2}(j) \frac{\sqrt{2}}{\sqrt{j}}\}$.
\end{restatable}
{The following Lemma \ref{"mylemma3"} shows how to select a suitable balanced parameter $\gamma$ of UCB score that ensures the best player is always in the candidate set.}


\begin{restatable}{lem}{lemmaiii}
\label{"mylemma3"}
Define the constant $C=\sqrt{2nT\log \left(\frac{T+\tau}{n}\right)}$. At each round $t>\tau$, let UCB balanced parameter $\gamma=2g_1(t)$ and assume $\Delta>g_1(T)C$, if $\alpha$ satisfies that $\alpha\geq \frac{\sqrt{2}\tau\sqrt{1+\log j}}{(\Delta-g_1(T)C)\sqrt{j}}$, then we have $x^*\in \mathcal{S}$ holds with probability at least $1-\frac{5}{T^2}$, where $j=\lfloor\frac{t-1}{\tau}\rfloor$, $\Delta$ is the difference between ratings of optimal player $x^*$ and sub-optimal player $x^\prime$. Recall $x^*=\argmax_{x\in[n]} r^*_x$, and define $x^\prime=\argmax_{x\in[n]/{x^*}}r^*_x$, $\Delta= r^*_{x^*}- r^*_{x^\prime}$.
\end{restatable}

{Lemma} \ref{"mylemma3"} shows that if we properly select UCB balanced parameter $\gamma$ and parameter $\alpha$ which describes objective function $l_{j,\tau}$ as a $\alpha$-strongly convex, then it is promised that the best player $x^*$ is in candidates set $\mathcal{S}$ with high probability.  
{This property is helpful for the top-$1$ identification because the candidate set $\mathcal{S}$ will become tighter with the time, and $x^*$ always in $\mathcal{S}$, thus candidate set $\mathcal{S}$ only contains  $x^*$ eventually. }
Together we are ready to present our main results in Theorem \ref{"thm:Elo-SGD-UCB"}.
\begin{mythm}\label{"thm:Elo-SGD-UCB"}
We run our \alg \ref{alg:dueling-sgd-UCB} to get a sequence of arm-pair, and let the learning rate parameter $\alpha\geq \max\{c_3, \frac{\sqrt{2}\tau\sqrt{1+\log j}}{(\Delta-g_1(T)C)\sqrt{j}}\}$ with assumption that $\Delta>g_1(T)C$, the balanced parameter $\gamma=2g_1(t)$, there exists two positive parameter $C_1,C_2$ such that if the batch size $\tau$ is chosen as \eq \eqref{"eq:tau"},
then we have the cumulative regret satisfies that:
\begin{equation*}
	 \resizebox{1\columnwidth}{!}{
	 $
	\begin{aligned}
    R(T)\leq \tau*\Delta_{max}+(2+\tau) g_1(T) \sqrt{2nT\log(\frac{2\tau+T}{n})}+4g_2(J)\sqrt{\tau T},
    \end{aligned}
    $}
\end{equation*}
with probability at least $1-\frac{10}{T}$, where $J=\lfloor\frac{T}{\tau}\rfloor$, $\Delta_{max}=\max_i r^*_i-\min_i r^*_i$, $g_1(T), g_2(J)$ is defined in \eq \eqref{eq:g1g2} and $C$ is a constant defined as $C=\sqrt{2nT\log \left(\frac{T+\tau}{n}\right)}$.  
\end{mythm}
Note that $\tau\sim O(\max\{n,\log T\})$ (\eq \eqref{"eq:tau"}), $g_{1}(T) \sim O(\sqrt{n\log T})$, $g_{2}(J) \sim O(\sqrt{\log T})$. Combining the above analysis, we have $R(T)\sim O(n\log T \sqrt{T})$ (or $\tilde{O}(\sqrt{T})$.
This regret upper bound is equivalent to that in \cite{saha2020regret} which employs MLE estimators.
However, our algorithm improves the efficiency in terms of memory and time. The memory cost is constant with respect to $T$ while MaxInP's memory cost is linear in the time horizon $T$.
{The time complexity of our MaxIn-Elo is $O(n^2T)$, while MaxInP's time complexity is $O(nT^2+n^2T)$.}
 {See Table \ref{tab:complexity} for a detailed comparison.}
Detailed proofs are referred to Appendix.



\section{Experiments}\label{sec:experiments}


We consider the following two batteries of experiments to evaluate the performance of our algorithms in the scenarios of transitive and intransitive real world meta-games. Ablation studies of parameter $\gamma$, dimension of mElo and the batch size $\tau$ can be found in Appendix. 

\subsection{Baselines}
\textbf{Random}: The pairwise matching scheme of the classical Round-robin \citep{rasmussen2008round} tournament is based on random sampling. We construct  a simple baseline that randomly select a pair from all ${n*(n-1)}/{2}$ pairs with replacement. After sampling a pair, we use the Elo/mElo model to update the ratings. 

\textbf{RG-UCB} \citep{rowland2019multiagent}: This algorithm adopts a pure exploration sampling scheme, which uniformly samples a pair from the set containing pairs that need to be estimated. And the stopping condition $C(\delta)$ controls the total number of comparisons of each pair, where $\delta$ is a hyper parameter deciding the confidence level of estimated competitive results.

\textbf{DBGD} \citep{yue2009interactively}: This dueling bandits algorithm is popular in ranking tasks when only pair-wise binary feedback is available. 
It maintains one winning arm at each round, and randomly synthesizes a gradient to obtain the opponent arm in the contextual bandit setting. 
In our feature free setting, this is equivalent to randomly selecting a player as the opponent. 

\textbf{$\alpha$-IG} \citep{rashid2021estimating}: This is an active sampling algorithm used for estimating the $\alpha$-rank \citep{omidshafiei2019alpha}. This algorithm  selects a pair with largest information gain at each round. In the transitive case, the top player has an $\alpha$-rank score  equal to $1$. Due to the high computation cost at each round ( computing $\alpha$-rank for 80000 times in a $4\times4$ game), we only compare with it in a $4 \times 4$ transitive game: the `2 Good, 2Bad' game given by $\alpha$-IG.

\textbf{MaxInP} \citep{saha2020regret}: This algorithm is for the generalized linear contextual dueling bandits problem, in which arms are represented as feature vectors. 
And it uses MLE to estimate model parameters $\theta$ relying on all historical pairwise comparisons at each round $t$. This algorithm calculates a candidate set containing advanced arms and pulls an arm-pair with the largest uncertainty. In order to fit their model, each player is described as a one-hot vector, and the estimated parameters $\theta$ correspond to players' ratings in our setting.

\textbf{MaxIn-Elo}: Our first algorithm adopts dueling bandits to adaptively sample pairs for Elo rating update in \eq \eqref{"eq:elo_hat"}.
The aim of our MaxIn-Elo is to identify the advanced players gradually, and to minimize the cumulative regret described in \eq \eqref{"eq:cumu reg"}
simultaneously.
    
\textbf{MaxIn-mElo}: Our second algorithm tames the intransitive scenarios. Different to $\ourmethod$, there is an extra vector $ c$ to capture intransitive relationship in competition outcome prediction. {The dimension of $ c$ is set to 8 in experiments.}
For the MaxIn-mElo algorithm, we hope to identify players with superior mElo ratings and to minimize cumulative regret on mElo ratings.

\subsection{Experiments setting}
\paragraph{Real world games}
We do our experiments on twelve real-games released by  \citet{czarnecki2020real}, most of which are implemented on the OpenSpiel framework \citep{lanctot2019openspiel}.
The six  games used for evaluating Elo are Triangular game,  Transitive game, Elo game, and three noisy variants of Elo games. 
The first three are transitive games; the three variants of Elo game are Elo games with additive Gaussian noises.
The six intransitive games used for evaluating mElo are Kuhn-poker, AlphaStar, tic\_tac\_toe, hex, Blotto and 5,3-Blotto game.

The intransitivity of games can be revealed by sink strongly connected components (SSCCs) \citep{omidshafiei2019alpha}, which 
is a set of strategies that cannot be defeated by external strategies and all internal strategies become a circle, such as Rock, Paper, Scissors.
The statistics of these games is shown in \tables 2 in Appendix. 

\paragraph{Metrics}
Except the cumulative regret defined in \eq \eqref{"eq:cumu reg"}, 
we introduce three other metrics for   Reciprocal Rank (RR), Normalized Discounted Cumulative Gain (NDCG), and Hit Raio (HR). 
RR is used for the results on generating top-1 players.
NDCG and HR report discrete performance for top-1 performance and are thus used in top-$k$ results.


Reciprocal Rank (RR) \citep{lambdarank} give the reciprocal of predicted ranking of the best player $x^*$. Define
$$RR=\frac{1}{R(x^*)},$$ where $R(x)$ returns the ranking of player $x$ relying on currently predicted ratings $\bar{r}$. Larger RR corresponds to better performance on the top-1  player identification.

$\text{Hit Ratio}@K$ \citep{he2015trirank} is defined as the ratio of the predicted top-$k$ that belongs to the true top-$k$. Since hit ratio does not consider the positions of correct predictions, 
{we also adopts NDCG \citep{lambdarank} which assigns higher importance  to results at top ranks. }
{Normalized Discounted Cumulative Gain (NDCG) is widely used in the evaluation of rankings and the NDCG@k measures the importance of predicted top-$k$ players. It is given by 
$$\text{NDCG}@K=\frac{1}{N_K}\sum_{i=1}^K\frac{2^{l(d_i)}-1}{\log(i+1)},$$ 
{where $N_k$ is a normalizer to ensure that the perfect ranking would result in $\text{NDCG}@K=1.$} $d_i$ denote the index of predicted i-th player, and $l(x)\in\{0, 1\}$ is the relevance level about top-k identification, we set $l(x)=1$ if player $x$ in true top-k otherwise 0.}

\paragraph{Parameters setting}
For Random, DBGD, and RG-UCB baseline, we perform a grid search for the initial step size $\eta$ in the range $\{0.01, 0.05, 0.1, 0.5, 1, 5, 10\}$.
For RG-UCB, stopping confidence $\delta=0.2$. 
For MaxInP, we tune the UCB balanced parameter $\gamma \in \{0.2, 0.4, 0.6, ... , 2.0\}$. 
For $\ourmethod$ and $\ourmethodii$, we tune the initialized learning rate $\eta \in\{0.01, 0.05, 0.1, 0.5, 1, 5, 10\}$, and the learning rate at batch $j$ is set as $\frac{\eta}{j}$. And the UCB balanced parameter $\gamma \in \{0.2, 0.4, 0.6, ... , 2.0\}$.
The batch size $\tau$ of MaxInP, MaxIn-Elo and MaxIn-mElo is set to $0.7*n$. When baselines ues mElo model to calculate ratings, we set the dimension of the extra vector $c$ as 8. We use the parameters that report the best performance for $\alpha$-IG. 
We repeat experiments 5 times with different random seeds and plot the averaged performance with standard deviations. 

All experiments were run in a single x86\_64 GNU/Linux machine with 256 AMD EPYC 7742 64-Core Processor and 2 A100 PCIe 40GB GPU. We use sklearn(0.24.2) to solve the MLE. 

\subsection{Results}

\begin{figure}
    \centering
    \includegraphics[width=0.9\linewidth]{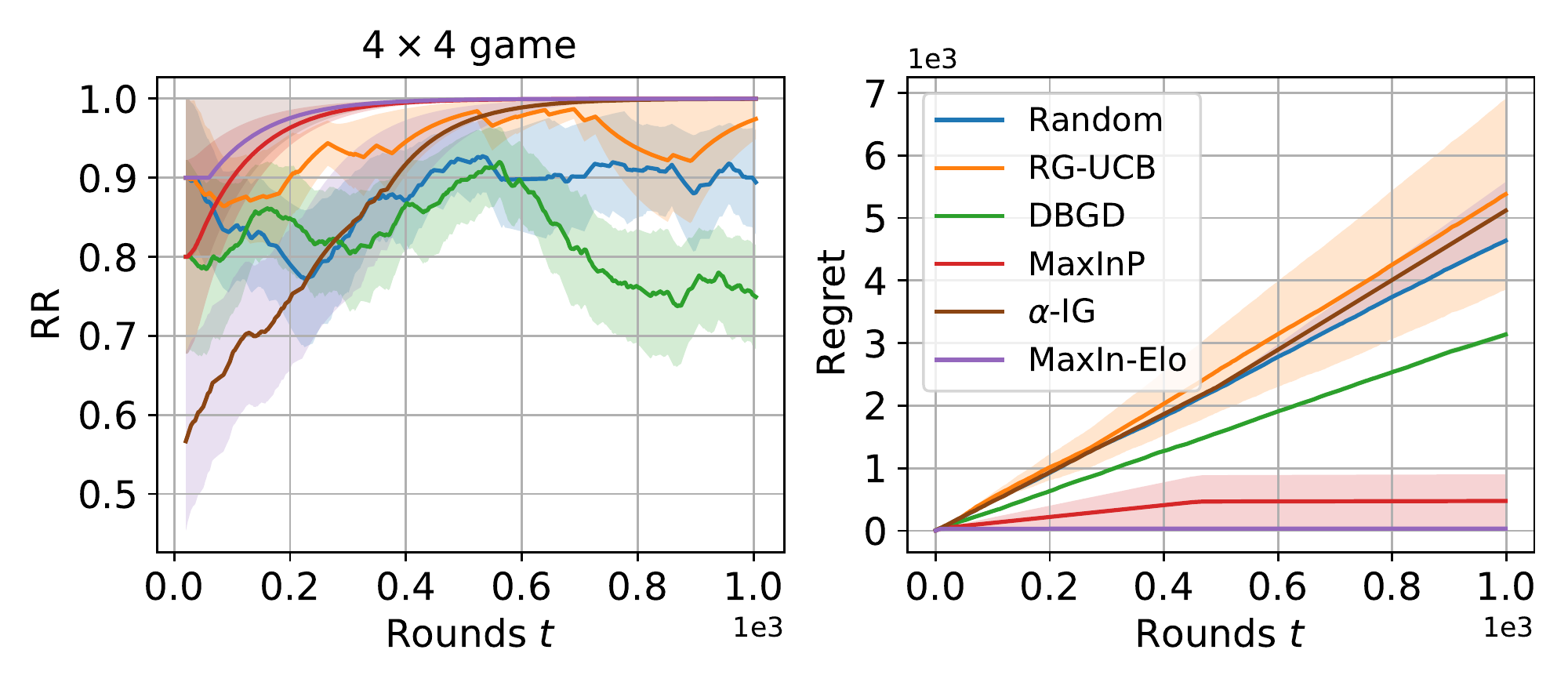}
    \caption{Results on $4\times 4$ game (2 Good 2 Bad).}
    \label{fig:alphaIG}
\end{figure}
\begin{figure}
    \centering
    \begin{subfigure}
        \centering
        \includegraphics[width=0.96\linewidth]{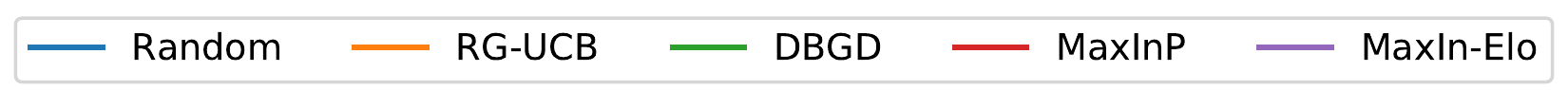}
    \end{subfigure}\\
    \begin{subfigure}
        \centering
        \includegraphics[width=0.46\linewidth]{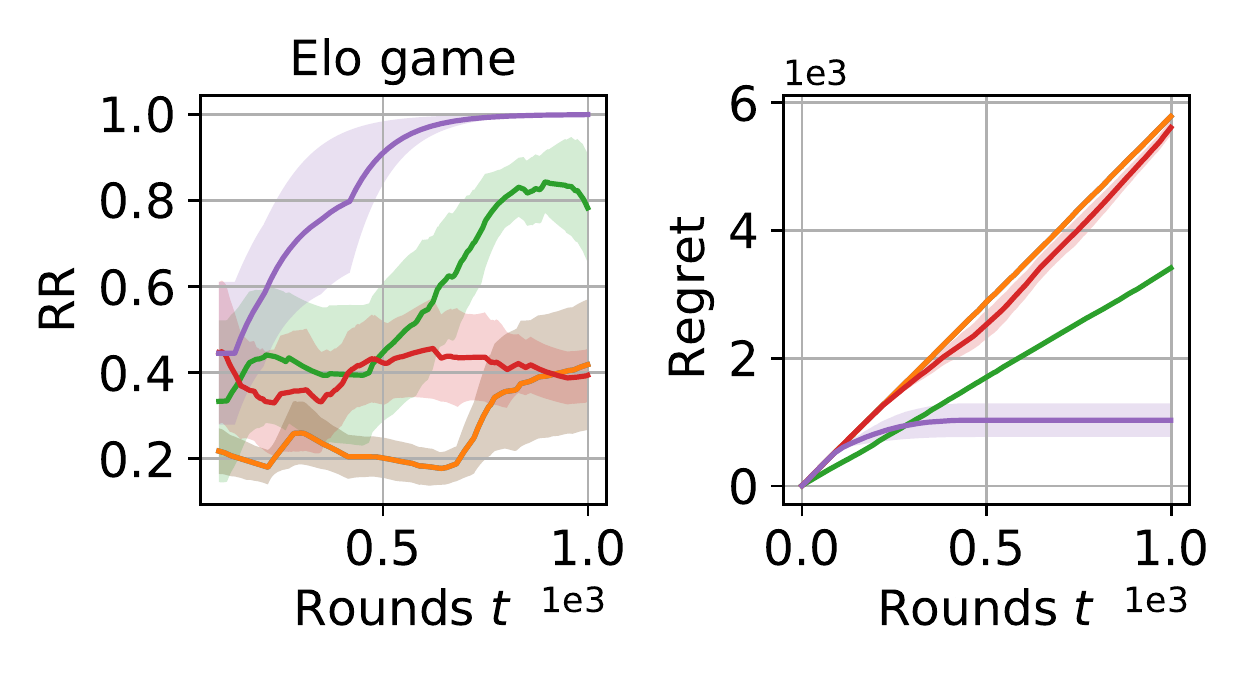}
    \end{subfigure}
    \begin{subfigure}
        \centering
        \includegraphics[width=0.46\linewidth]{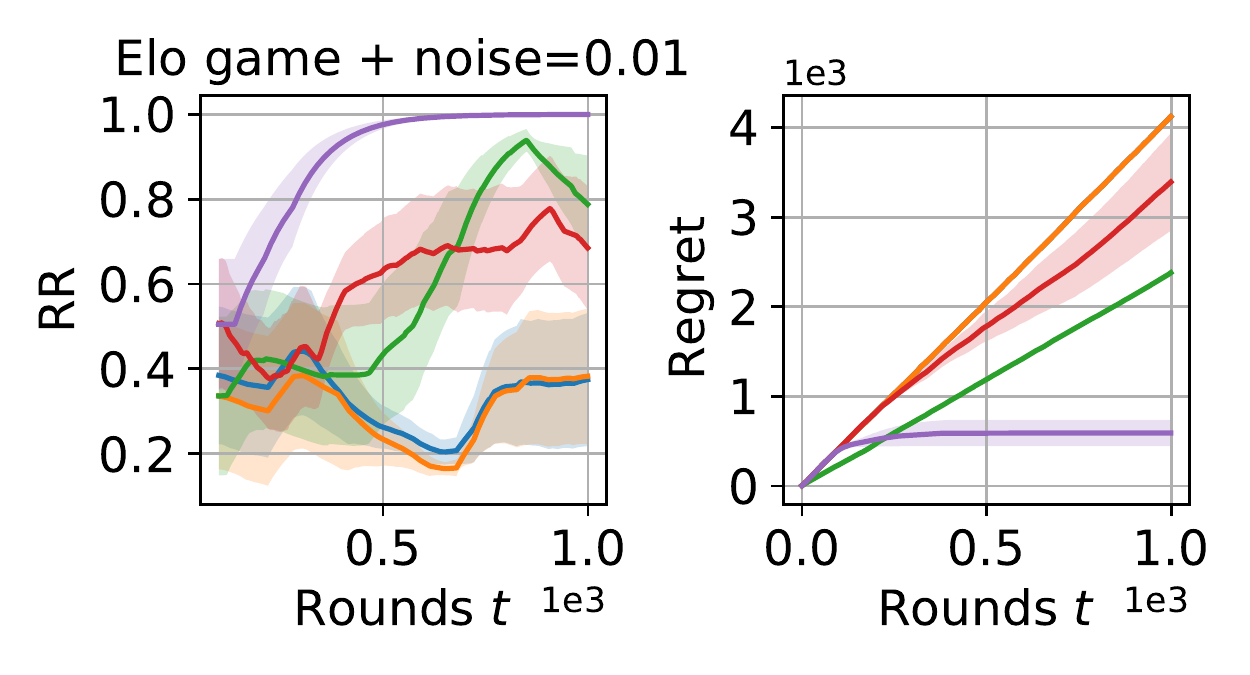}
    \end{subfigure}\\
    \begin{subfigure}
        \centering
        \includegraphics[width=0.45\linewidth]{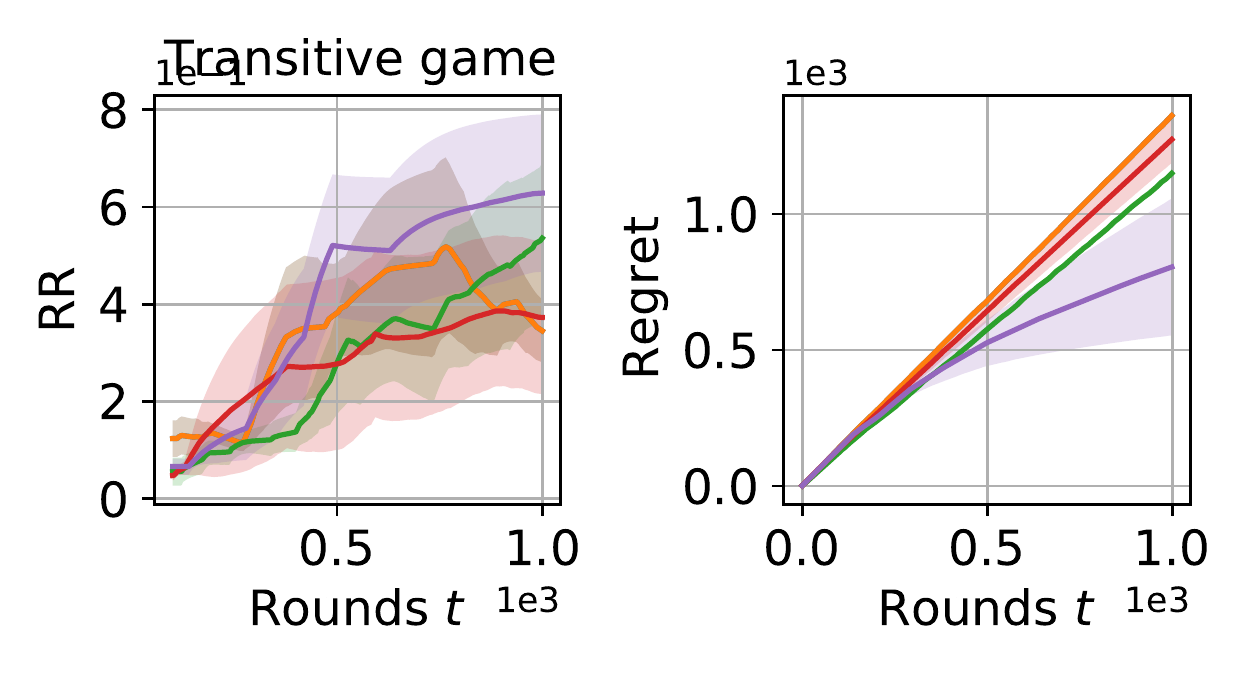}
    \end{subfigure}
    \begin{subfigure}
        \centering
        \includegraphics[width=0.46\linewidth]{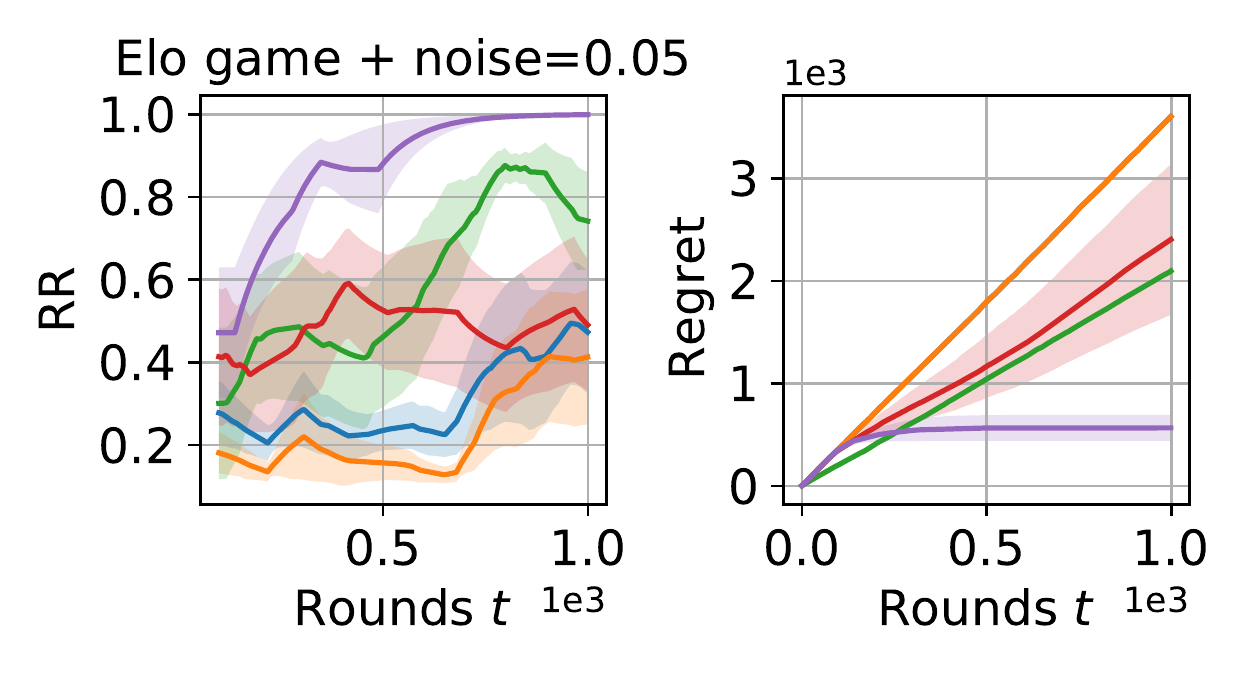}
    \end{subfigure}\\
    \begin{subfigure}
        \centering
        \includegraphics[width=0.46\linewidth]{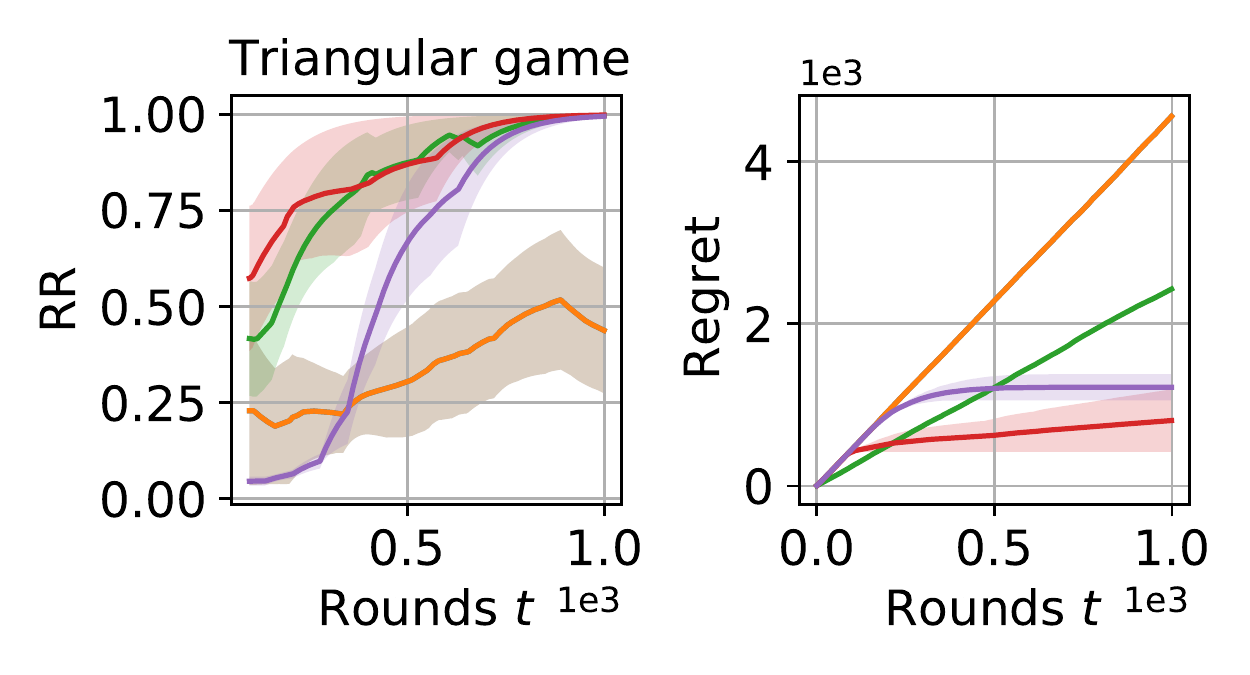}
    \end{subfigure}
    \begin{subfigure}
        \centering
        \includegraphics[width=0.46\linewidth]{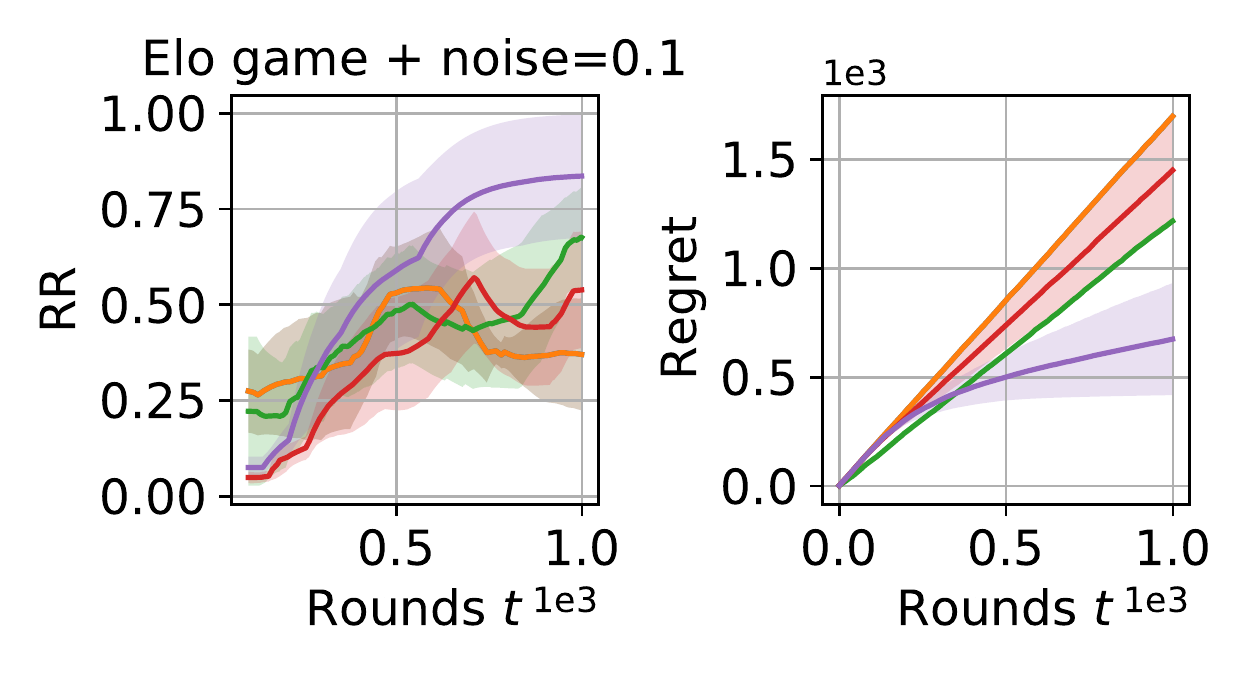}
    \end{subfigure}
    \caption{Results of Elo on transitive games.}
    \label{tab:tran-real-game}
\end{figure}
\setlength{\belowcaptionskip}{-0.5cm}




{\fig \ref{fig:alphaIG}, \ref{tab:tran-real-game}, \ref{tab:intran-real-game} } show the results of top-$1$ identification on 13 games. To ensure a fair comparison between all baselines, we perform a grid search to select parameters with the best RR performance for each random seed. If the winning probability matrix can be fitted into the Elo model, then we calculate the true ratings through \eq \eqref{"eq:elo_hat"},  otherwise we use the mElo ratings as the true ratings through \eq \eqref{"eq:melo_qhat"}.

\textbf{Evaluation of $\ourmethod$}
\fig \ref{fig:alphaIG} shows the results on a $4\times 4$ transitive game. $\ourmethod$ has the highest convergence rate on both RR and cumulative regret metrics, and $\ourmethod$ has the lowest cumulative regret close to $0$. 
As shown in \fig \ref{tab:tran-real-game}, $\ourmethod$  significantly outperforms all other baselines on five games and achieves similar performance on Triangular game.
Regarding the RR metric, 
$\ourmethod$ can converges to $1$ on four games. Even on Transitive game and Elo game + noise=$0.1$, RR scores as up to $0.6$ and $0.8$ respectively, which indicates that the rank of the top player no more than $2$. Thus we think $\ourmethod$ has the ability to effectively identify the top player . 
On the Elo game, Elo game + noise=$0.01$, and Elo game + noise=$0.01$, the cumulative regret are closed to convergence at around $500$ rounds. 
When the cumulative regret meets convergence, the candidate optimal set $\mathcal{S}$ only contains the top player, and no regret increasing.  

Different from the other 5 stochastic games, Triangular game is a deterministic game with all winning probabilities are equal to $1$ or $0$, thus it is easy to evaluate. For DBGD baseline, it maintains the current best player and randomly selects an opponent, so it could find the best player more quickly, but has a large cumulative regret because of randomly selected opponents.  


\begin{figure}
    \centering
    \begin{subfigure}
        \centering
        \includegraphics[width=0.96\linewidth]{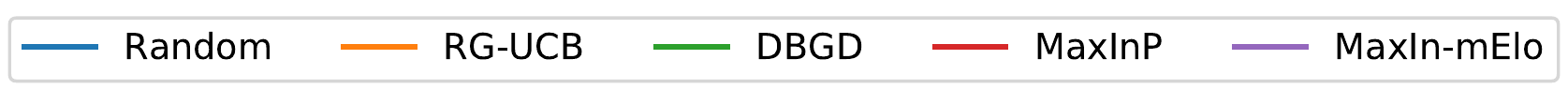}
    \end{subfigure}\\
    \begin{subfigure}
        \centering
        \includegraphics[width=0.46\linewidth]{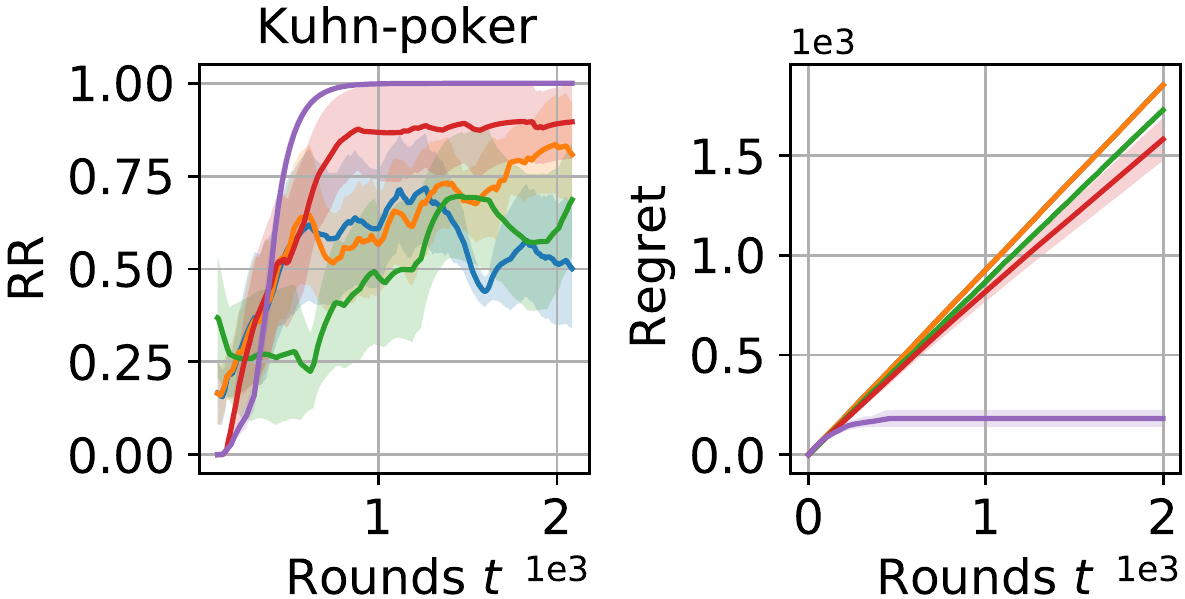}
    \end{subfigure}
    \begin{subfigure}
        \centering
        \includegraphics[width=0.46\linewidth]{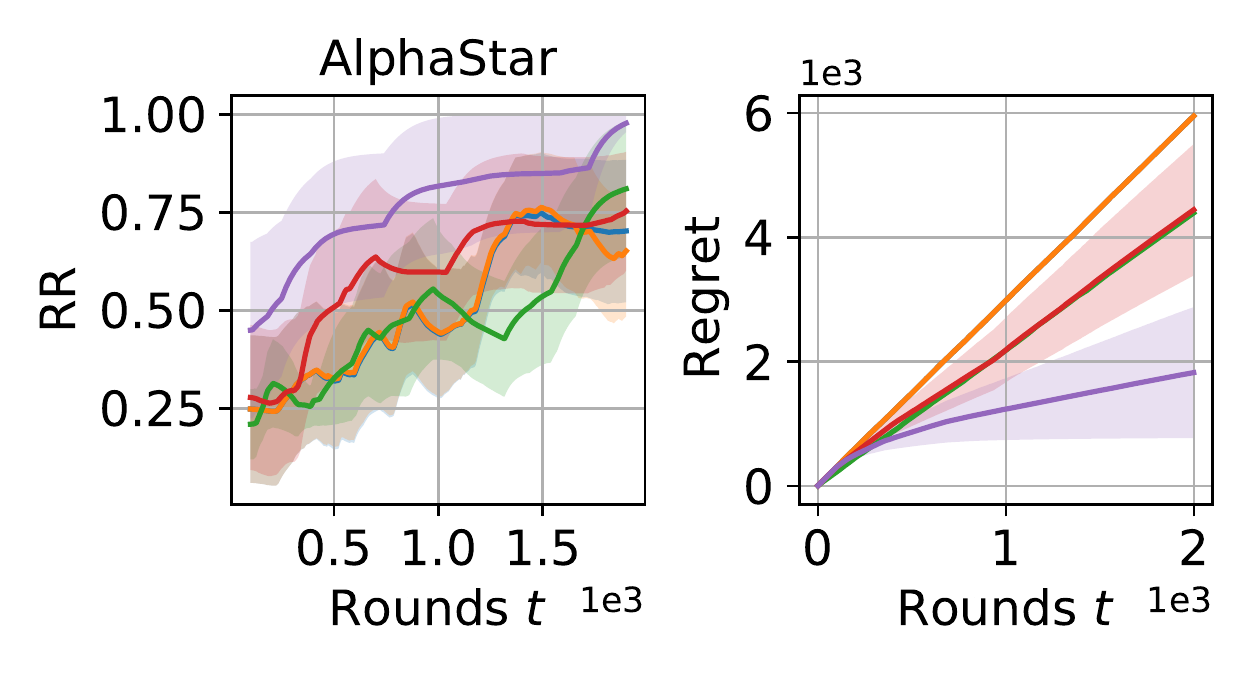}
    \end{subfigure}\\
    \begin{subfigure}
        \centering
        \includegraphics[width=0.46\linewidth]{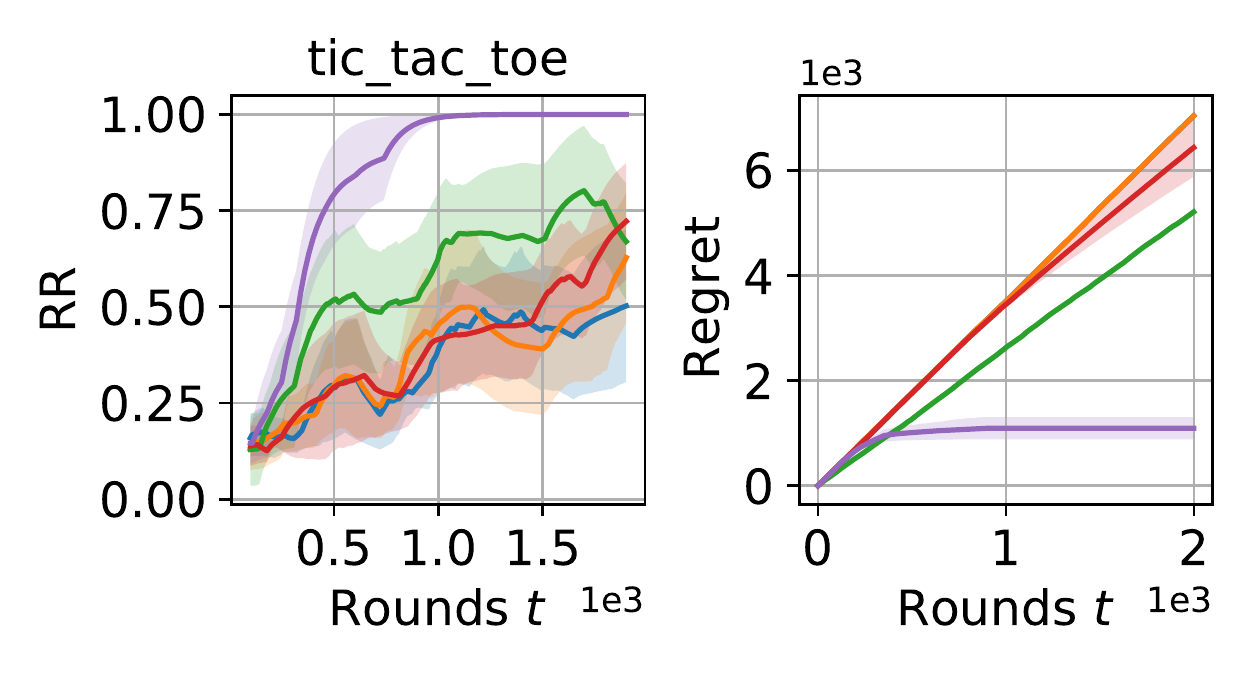}
    \end{subfigure}
    \begin{subfigure}
        \centering
        \includegraphics[width=0.46\linewidth]{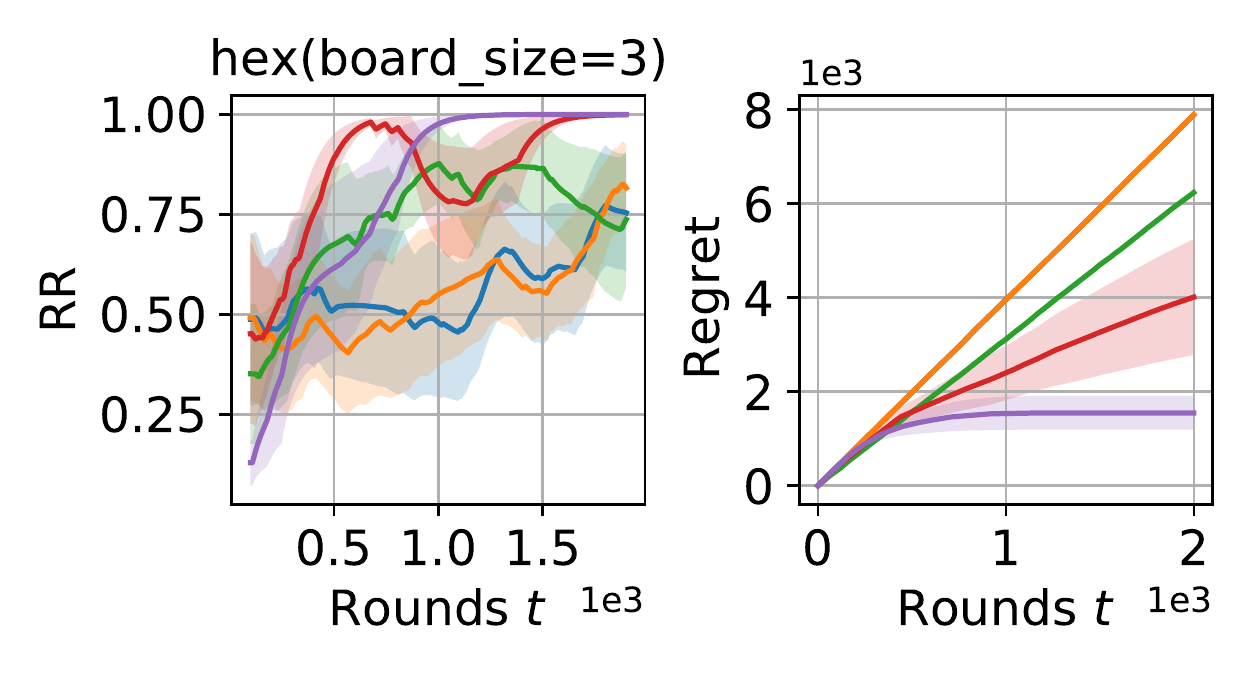}
    \end{subfigure}\\
    \begin{subfigure}
        \centering
        \includegraphics[width=0.46\linewidth]{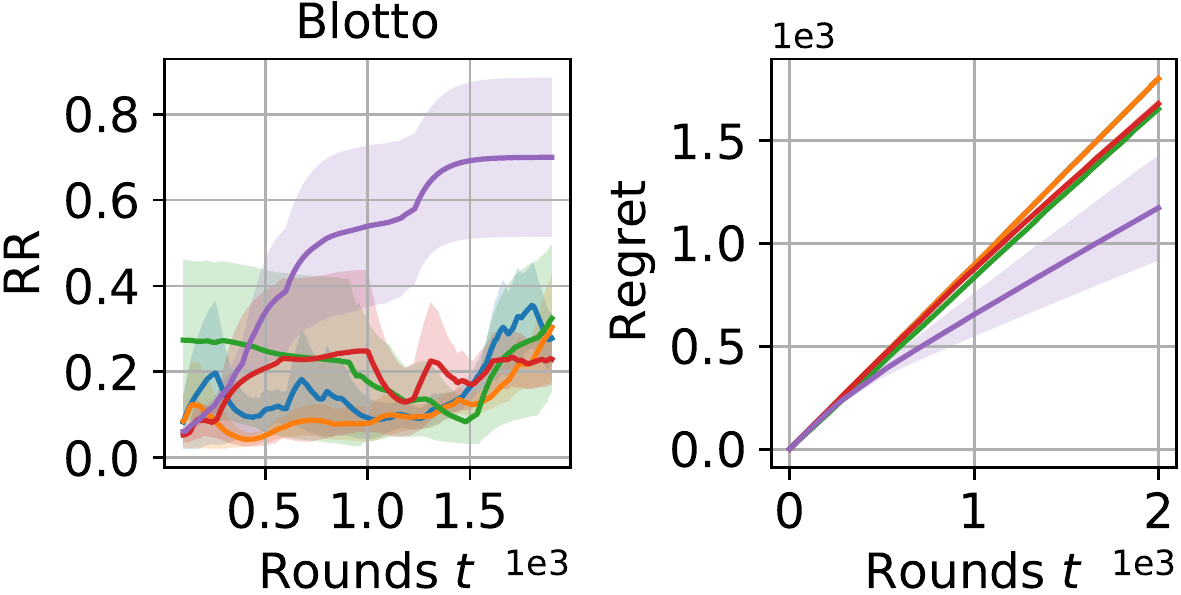}
    \end{subfigure}
    \begin{subfigure}
        \centering
        \includegraphics[width=0.46\linewidth]{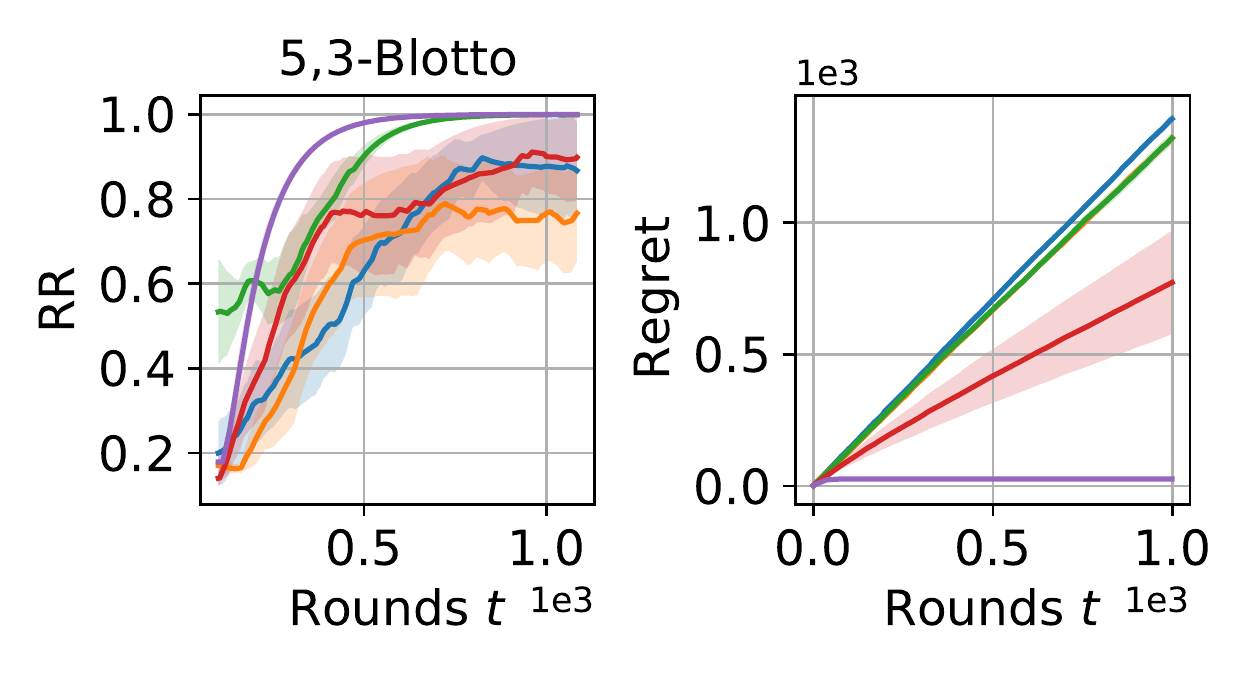}
    \end{subfigure}
    \caption{Results of mElo on intransitive games.}
    \label{tab:intran-real-game}
\end{figure}
\setlength{\belowcaptionskip}{-0.5cm} 

\textbf{Evaluation of $\ourmethodii$}
\fig \ref{tab:intran-real-game} shows the results of baselines on six real-world intransitive games. MaxInP is based on the Elo model for it is a special generalized linear model only with rating parameter $ r$ without cyclic vector parameter $c$, but all other baselines are based on the mElo model. 
As the \fig \ref{tab:intran-real-game} shows, $\ourmethodii$ has the lowest cumulative regret and the highest RR on all six games. With regard to the RR, $\ourmethodii$ can be up to $1$ on all games except for Blotto. One possible reason why $\ourmethodii$ cannot be up to $1$ on Blotto may be that its size of top SSCC is very large.
The other reason is that we use the low-rank approximation of the probability matrix's rotation on the mElo model. Although we misidentified the top-$1$ player, we are still better than all other baselines.

\paragraph{Results of Top-$k$ player identification}
\fig \ref{tab:tran-real-game-topk} gives the results of top-$k$ predictions on transitive games.
$\ourmethod$ and MaxInP both have a  parameter $\gamma$ used to balance exploration and exploitation, larger $\gamma$ can lead to a larger candidate set then lead to better top-$k$ performance. 
We keep other parameters fixed and run experiments with different $\gamma\in \{0.2, 0.4, 0.6, ..., 2.0\}$, and we report the performance of MaxInP and $\ourmethod$ under the best $\gamma$. 
\fig \ref{tab:tran-real-game-topk} shows that $\ourmethod$ has the best performance of the top-$1$ identification on all games, and it achieves the comparable performance of top-$k$ identification on most games. Results of different $\gamma$ can be found in Appendix.


\begin{figure}[t!]
    \centering
    \begin{subfigure}
        \centering
        \includegraphics[width=0.96\linewidth]{figs/legend.pdf}
    \end{subfigure}\\
    \begin{subfigure}
        \centering
        \includegraphics[width=0.46\linewidth]{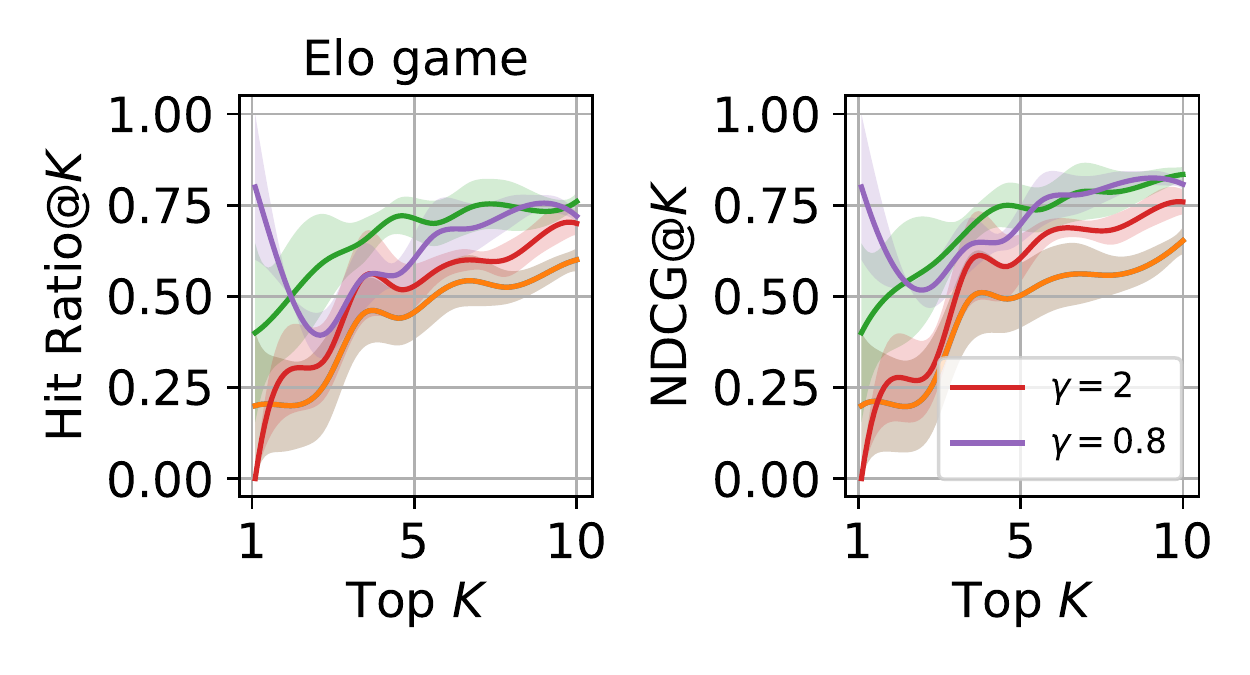}
    \end{subfigure}
    \begin{subfigure}
        \centering
        \includegraphics[width=0.46\linewidth]{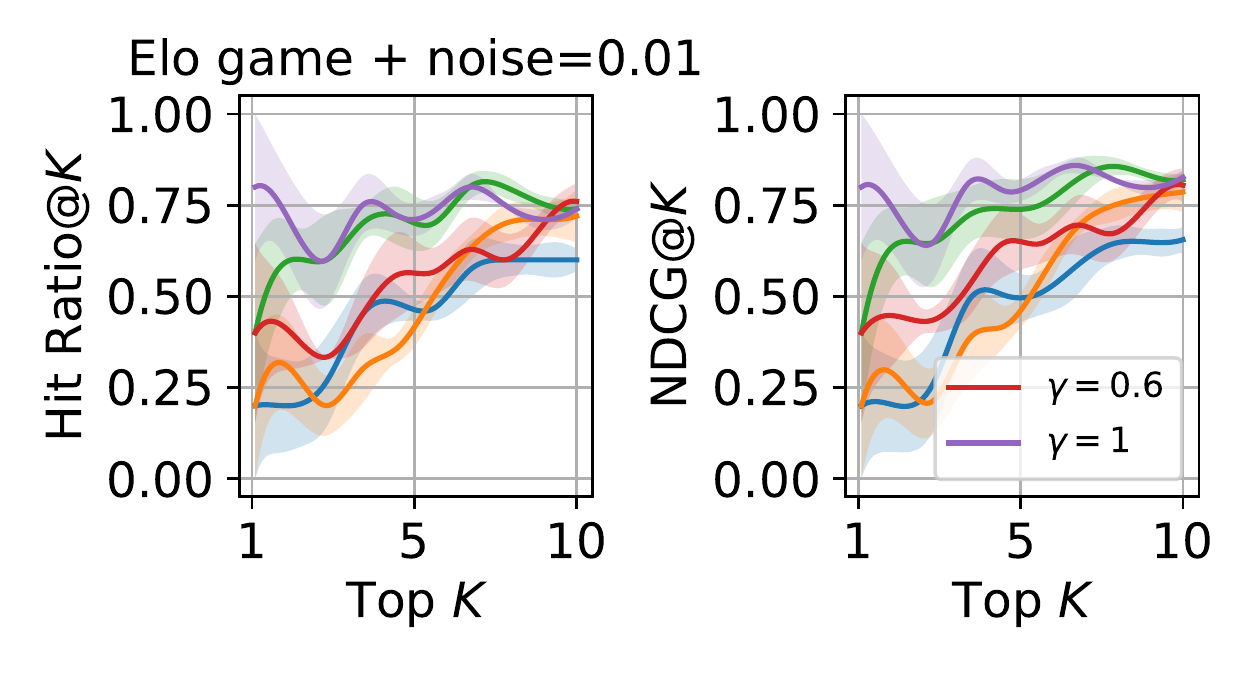}
    \end{subfigure}\\
    \begin{subfigure}
        \centering
        \includegraphics[width=0.45\linewidth]{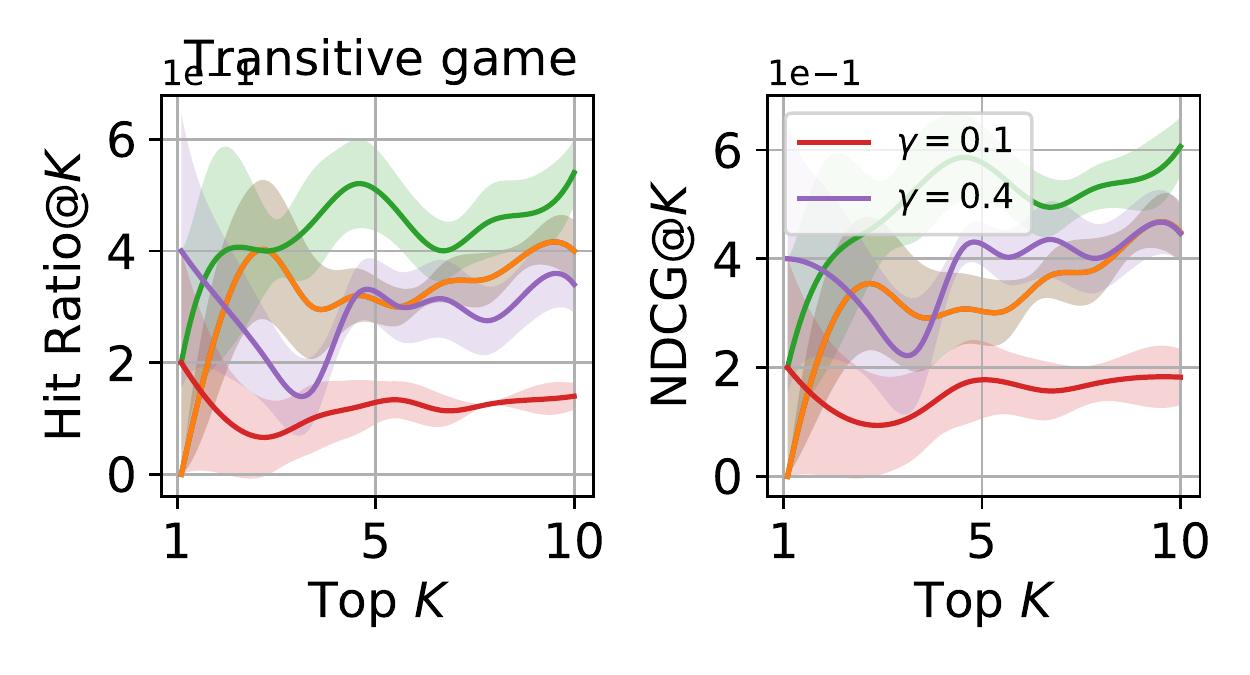}
    \end{subfigure}
    \begin{subfigure}
        \centering
        \includegraphics[width=0.46\linewidth]{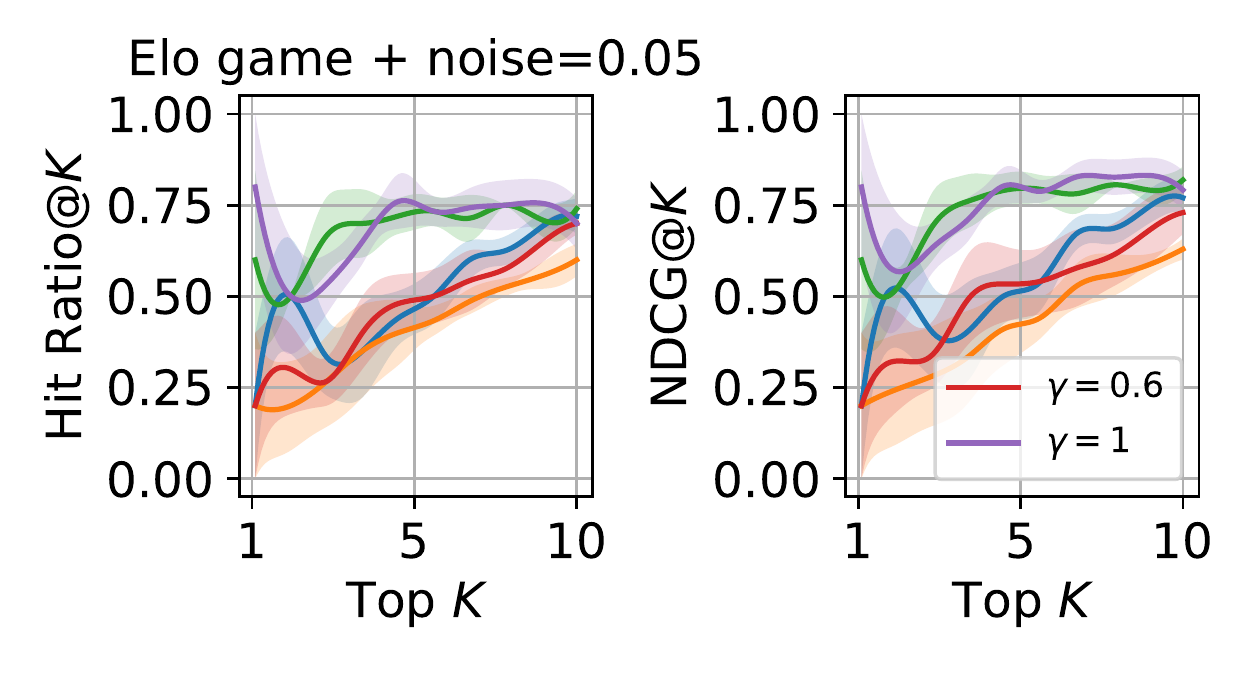}
    \end{subfigure}\\
    \begin{subfigure}
        \centering
        \includegraphics[width=0.46\linewidth]{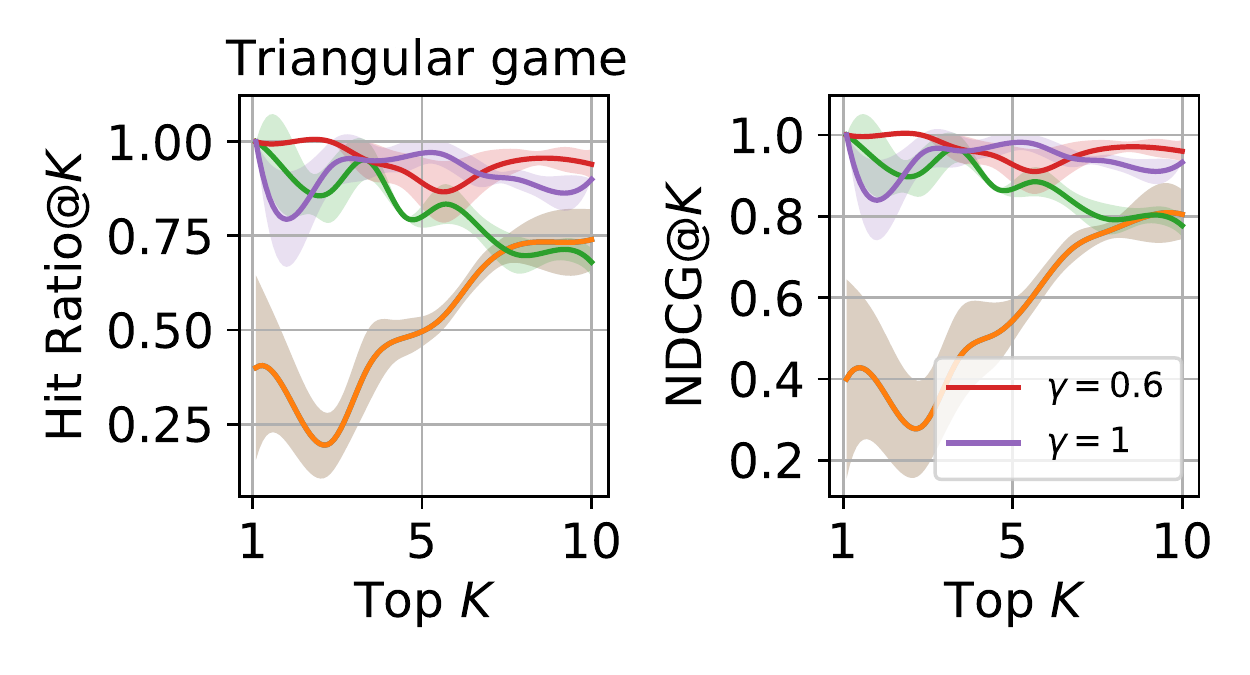}
    \end{subfigure}
    \begin{subfigure}
        \centering
        \includegraphics[width=0.46\linewidth]{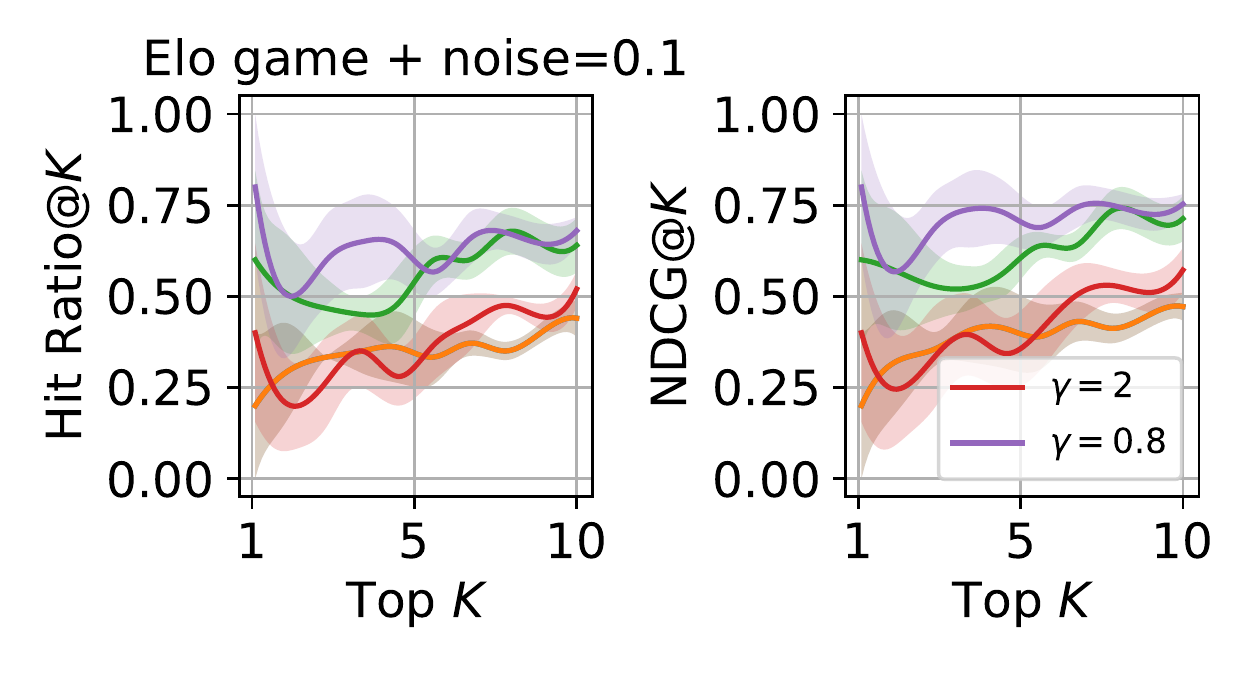}
    \end{subfigure}
    \caption{Results of Top-$k$ player identification on transitive games.
    {$\gamma$ in \textcolor{red}{red} and \textcolor{violet}{purple} indicates that reports best performance for MaxInP and $\ourmethod$ respectively.}}
    \label{tab:tran-real-game-topk}
\end{figure}
\setlength{\belowcaptionskip}{-0.5cm}

\section{Discussions}\label{sec:conclusion}
This work studied the problem of multi-agent evaluation with Elo ratings.  We have adopted an online match scheduling framework to improve the sample efficiency of the Elo rating system and its extension mElo for the intransitive settings. Both empirical and theoretical results justify that our algorithms can achieve higher sample efficiency and lower regret on most of the tasks.

We consider two limitations of this work. Firstly, the match outcome prediction in our algorithm is based on only ratings without considering features that describe players. Future work may consider adding features into the match prediction.
Secondly,  our algorithm focuses more on identifying the best player without being tailored for identifying top-$k$ players. Future work can consider active sampling that achieves better results on both top-1 and top-$k$ cases.

\section*{Ethical Statement}
This work proposes algorithms for online match scheduling that improve the efficiency in identifying top players in competitive games such as chess. While empirical studies in this work, which are based on AI agents, have demonstrated the superior gain of using our proposed methods, there is a caveat that our algorithms assume that the all players' skill levels remain unchanged throughout the repeated competition rounds. This assumption likely does not hold for human players whose playing strengths will be affected by energy consumption due to frequent matches. Therefore, extra caution needs to be taken when deploying our methods to schedule real-world competitions involving human players and an interesting research extension would be to model such performance strength changes explicitly in designing the match scheduling algorithms.

\balance
\bibliography{ref}

\onecolumn
\appendix
\section{Additional Details on Algorithms }\label{app:algorithms}

\subsection{Additional background on multi-dimensional Elo ratings (mElo)}
The winning probability matrix $P$ can be converted into an antisymmetric matrix $A=\sigma^{-1}(P)$ with $A_{ij}=\log\frac{P_{ij}}{1-P_{ij}}$. 

From the Hodge decomposition theory \cite{jiang2011statistical}, the antisymmetric matrix $A$ decomposes as $A=grad(r)+rot(A)$, {where $rot(A)_{ij}=\frac{1}{n}\sum_{k=1}^n A_{ij}+A_{jk}-A_{ik}$, capturing the intransitive relations,  $grad(r)=r \cdot\bm{1}^\top-\bm{1}\cdot r^\top$, indicating the combination gradient of divergence vector $r=\frac{1}{n}A\cdot1$.}
According to the Schur decomposition, mElo extends the Elo model by learning an low-rank($2k$) approximation of rotation matrix $rot(A)$, the cyclic component of $A$. 
Consider
$$A=grad(r)+rot(A)\approx grad(r)+C_{n\times2k}\Omega_{2k\times2k}C_{n\times2k}^\top,$$ 
where $\Omega_{2k\times2k}=\sum_{i=1}^k(e_{2i-1}e_{2i}^\top-e_{2i}e_{2i-1}^\top).$
By learning a $2k$-dimensional vector $c_x$ and a rating $r_x$ per player, the win-loss prediction for mElo$_{2k}$ is defined as:
\begin{equation}\label{"eq:melo_qhat_app"}
\hat{p}_{xy}=\sigma\left(r_{x}-r_{y}+{c}_{x}^{\top} \cdot {\Omega}_{2 k \times 2 k} \cdot {c}_{y}\right).
\end{equation}
Notice that the Elo model uses $k=0$.

Based on the probability estimation function in \eq \eqref{"eq:melo_qhat_app"}, the update functions equivalent to \eq \eqref{eq:elo-grad-descent} on mElo is:
\begin{align}\label{eq:melo_gradient}
   r_{x}^{t+1} &=r_{x}^{t}+\eta \cdot\left(o_{xy}^{t}-\hat{p}_{xy}^{t}\right), \\
      c_x^{t+1}(2i-1) & =c_x^{t}(2i-1)+\eta \cdot\left(o^t_{xy}-\hat{p}_{xy}^{t}\right)c_y^{t}(2i), \\
     c_x^{t+1}(2i) &  = c_x^{t}(2i)-\eta \cdot\left(o_{xy}^{t}-\hat{p}_{xy}^{t}\right)c_y^{t}(2i-1).
\end{align}

{Comparing to the Elo ratings,  mElo$_{2k}$ adopts a different win-loss prediction function by introducing a vector $c_x$, but still follows the  gradient descent update. }
 
\subsection{Algorithm for $\ourmethodii$}
\alg \ref{alg:melo-sgd-UCB} gives the details of $\ourmethodii$. Our MaxIn-mElo algorithm can be obtained by modifying the loss function, gradient and UCB score function in our MaxIn-Elo algorithm. In order to solve the intransitive games, our MaxIn-mElo employs the mElo model, and considers both the transitive and cyclic components to calculate the UCB score. 
\begin{algorithm}[th!]
\caption{$\ourmethodii $: Dueling bandits with online SGD for mElo. }
\label{alg:melo-sgd-UCB}
\begin{algorithmic}[1]
    \REQUIRE $T, N, \tau, \alpha,\gamma$
    \ENSURE output $\theta$
    \STATE Randomly choose a pair to compare and record as $x_t,y_t, o_t$ for $t \in [\tau]$
    \STATE $V_{\tau+1}=\sum_{t=1}^\tau (e_{x_t}-e_{y_t})(e_{x_t}-e_{y_t})^T$
     \STATE Initialize $\hat{\bm r}_\tau $ by MLE as Line 3 in Alg. \ref{alg:dueling-sgd-UCB} and $c$.
    \STATE Maintain convex set $\mathcal{C}=\left\{r:\left\|r-\hat{r}_{\tau}\right\| \leq 2\right\}$
     \FOR{ $t=\tau+1, \tau+2, \dots, T$}
        \IF{ $t\% \tau =1$ }
        \STATE $j \leftarrow \lfloor(t-1)/\tau \rfloor $ and $\eta_j = \frac{1}{\alpha j}$
        \STATE 
        Calculate gradient $\nabla_{r}l_{j,\tau}(\tilde{r}_{j-1})$ according to \eq \eqref{eq:elo-grad-descent}
        \STATE  Update gradient $ \tilde{r}_{j} \leftarrow \prod_{\mathcal{C}}\left(\tilde{r}_{j-1}-\eta_{j} \nabla_{r}l_{j,\tau}(\tilde{r}_{j-1}) \right)$
        \STATE Compute $\bar{r}=\frac{1}{j} \sum_{q=1}^{j} \tilde{r}_{q}$
        \STATE Calculate gradient $\nabla_{C}l_{j,\tau}(\tilde{C}_{j-1})$ 
        \STATE Update gradient $ \tilde{C}_{j} \leftarrow \tilde{C}_{j-1}-\eta_{j} \nabla_{C}l_{j,\tau}(\tilde{C}_{j-1})$
        \ENDIF
        \STATE Define a candidate optimal set $\mathcal{S}=\{x\ |\ \bar{r}_x-\bar{r}_y+\bar{c}_x^T\bm \Omega \bar{c}_y+\gamma\
\left\|e_x-e_y\right\|_{V_{t}^{-1}}>0 ,\ \forall y \in[n]/\{x\}\}$
        \STATE Select a pair$(x_t,y_t)=\argmax_{(x,y) \in \mathcal{S}}\left\|e_x-e_y\right\|_{V_{t}^{-1}}$.
        \STATE Pull arms $(x_t, y_t)$ and observe $o_t(x_t, y_t)$
        \STATE Compute $V_{t+1}=V_{t}+(e_{x_t}-e_{y_t})(e_{x_t}-e_{y_t})^T$
        \ENDFOR
\end{algorithmic}
\end{algorithm}

\begin{table}[t!]
\centering
\caption{Statistics of twelve real world games. $|\text{Top SSCC}|$ indicates the size of the SSCC with best performance.} 
\begin{tabular}{ccccc}\\\toprule  
Game                         & \# Policies & $|\text{Top SSCC}|$ & Transitivity  \\ \hline
Transitive game              & 100 & 1  &Yes     \\
Triangular game              & 100 & 1  &Yes  \\
Elo game                     & 100 & 1  &Yes    \\
Elo game + noise=0.01         & 100 & 1 &Yes \\
Elo game + noise=0.05         & 100 & 1 &Yes  \\
Elo game + noise=0.1         & 100 & 1 &Yes  \\
Kuhn-poker                   & 64   & 64 &No    \\
AlphaStar                    & 100  & 1 &No \\
tic$\_$tac\_toe                & 100  & 2 &No  \\
hex(board\_size=3)          & 100 & 2 &No  \\
Blotto  & 100&99 &No              \\
5,3-Blotto      & 21 & 18 &No             \\

\hline
\end{tabular}
\label{tab:real-game-statistics}
\end{table}

\begin{figure}[th!]
    \centering
    \begin{subfigure}
        \centering
        \includegraphics[width=0.66\linewidth]{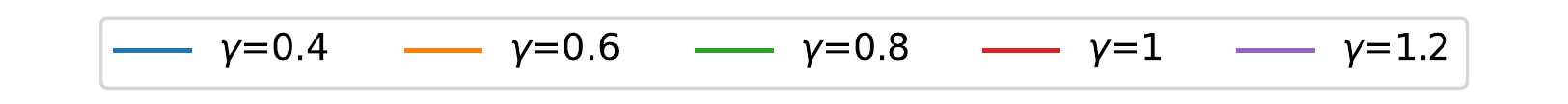}
    \end{subfigure}\\
    
    \begin{subfigure}
        \centering
        \includegraphics[width=0.32\linewidth]{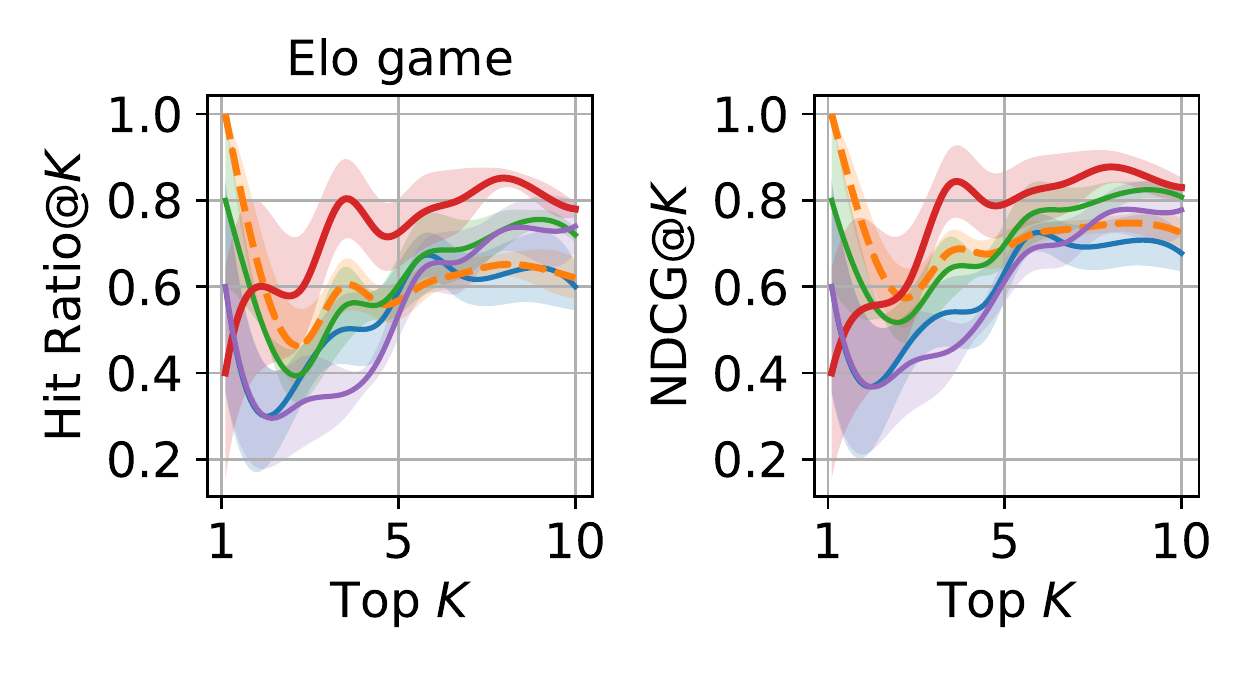}
    \end{subfigure}
    \begin{subfigure}
        \centering
        \includegraphics[width=0.32\linewidth]{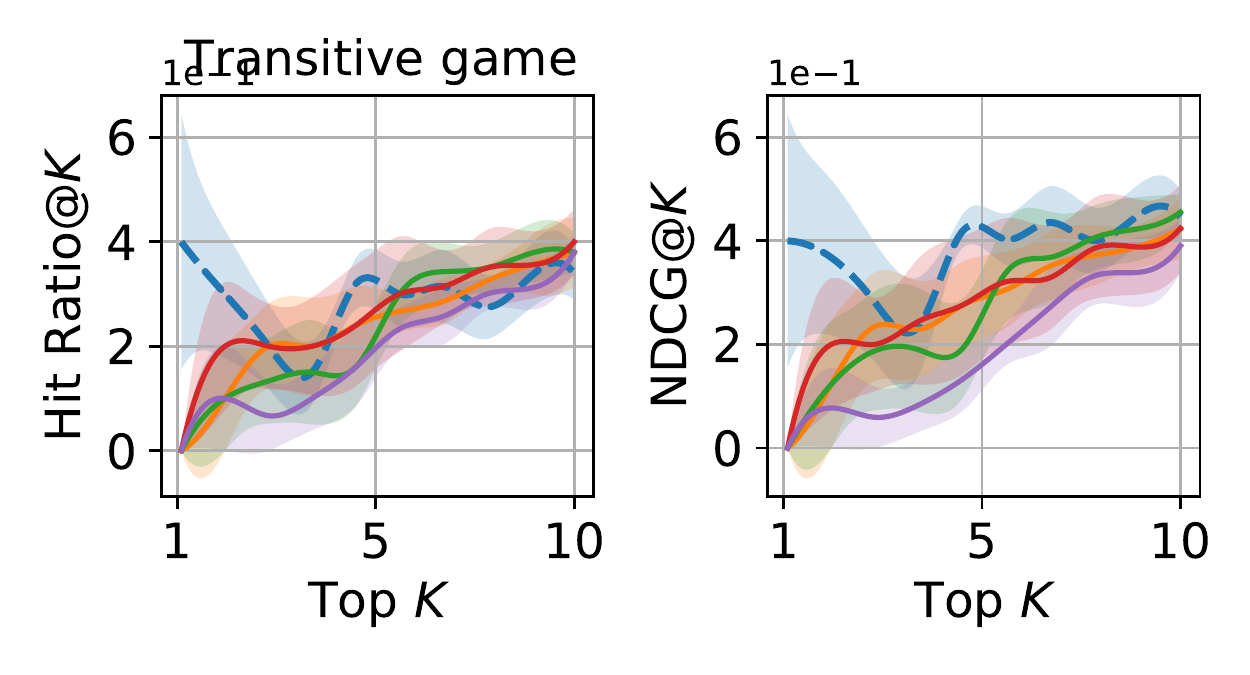}
    \end{subfigure}
    \begin{subfigure}
        \centering
        \includegraphics[width=0.32\linewidth]{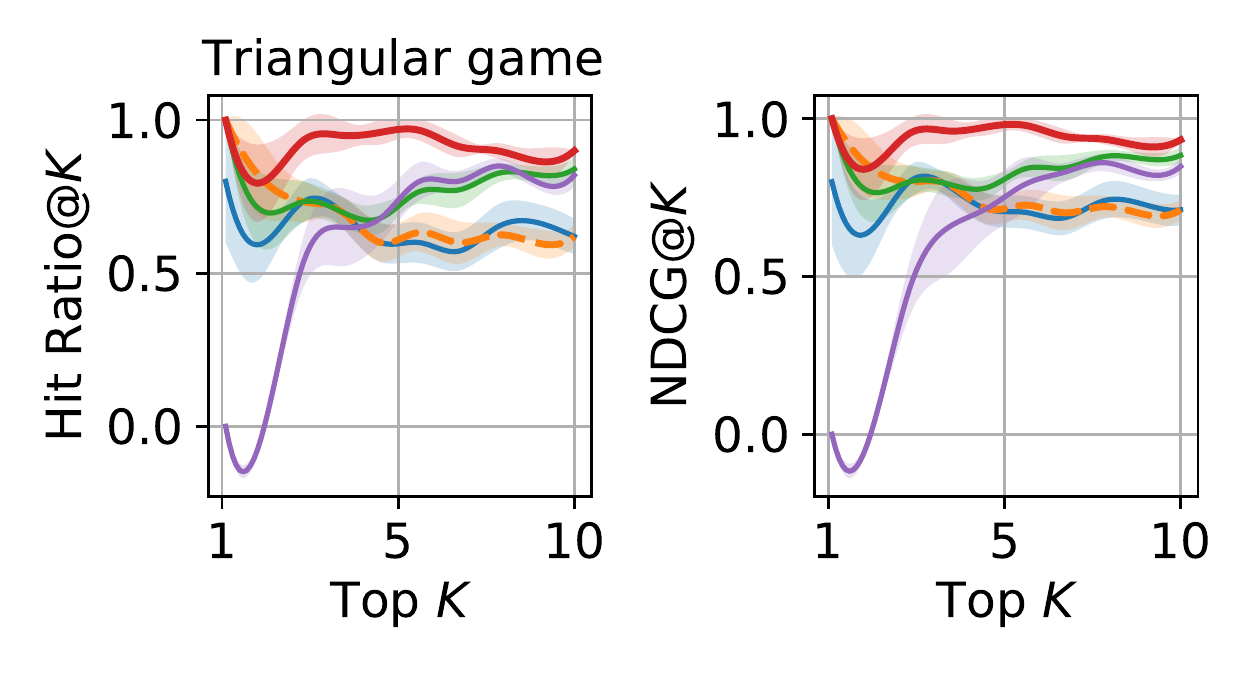}
    \end{subfigure}\\
    \begin{subfigure}
        \centering
        \includegraphics[width=0.32\linewidth]{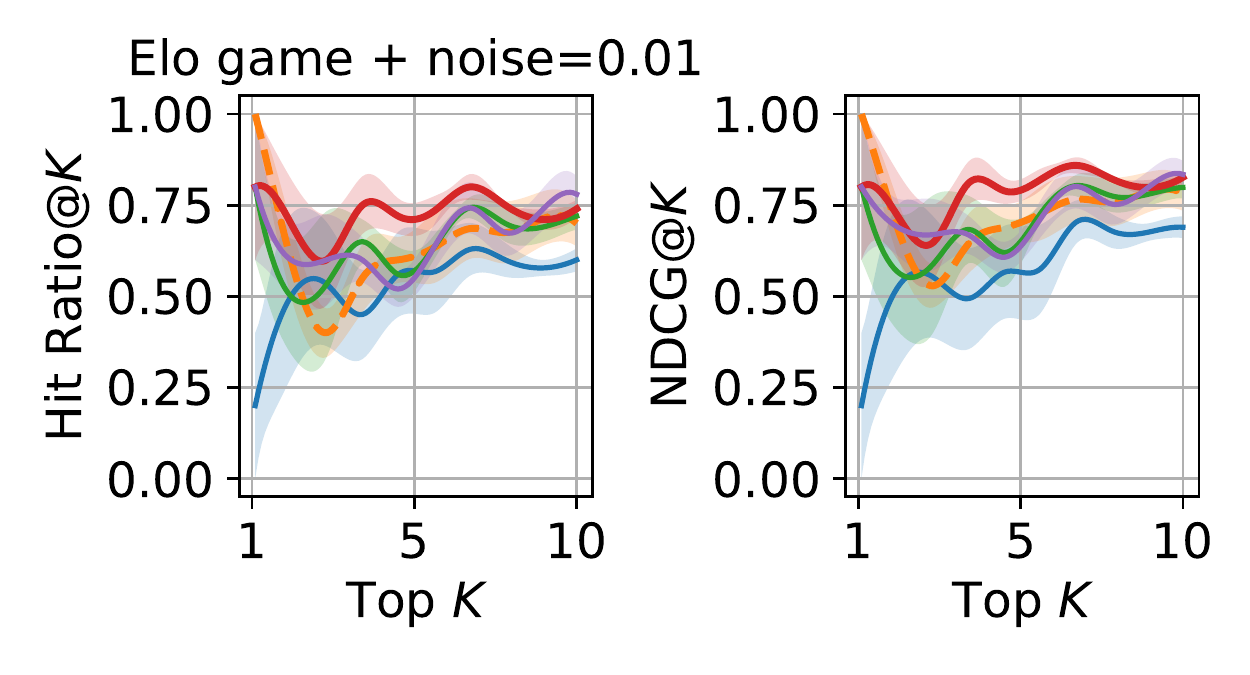}
    \end{subfigure}
    \begin{subfigure}
        \centering
        \includegraphics[width=0.32\linewidth]{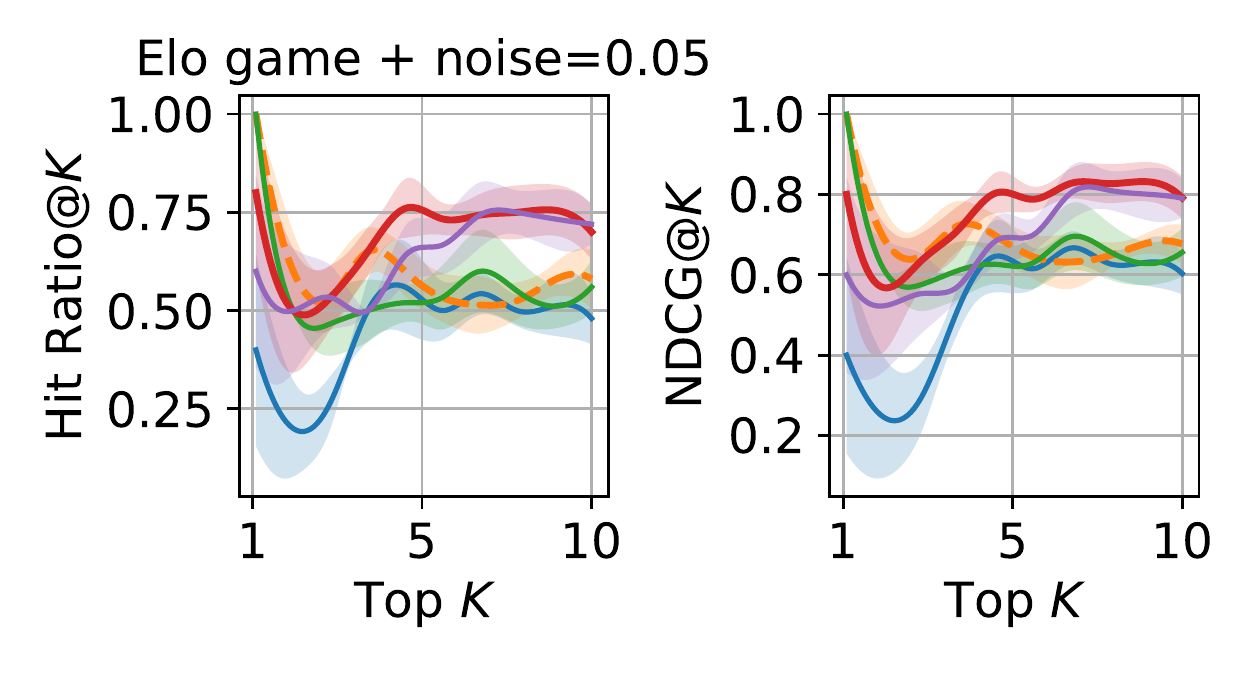}
    \end{subfigure}
    \begin{subfigure}
        \centering
        \includegraphics[width=0.32\linewidth]{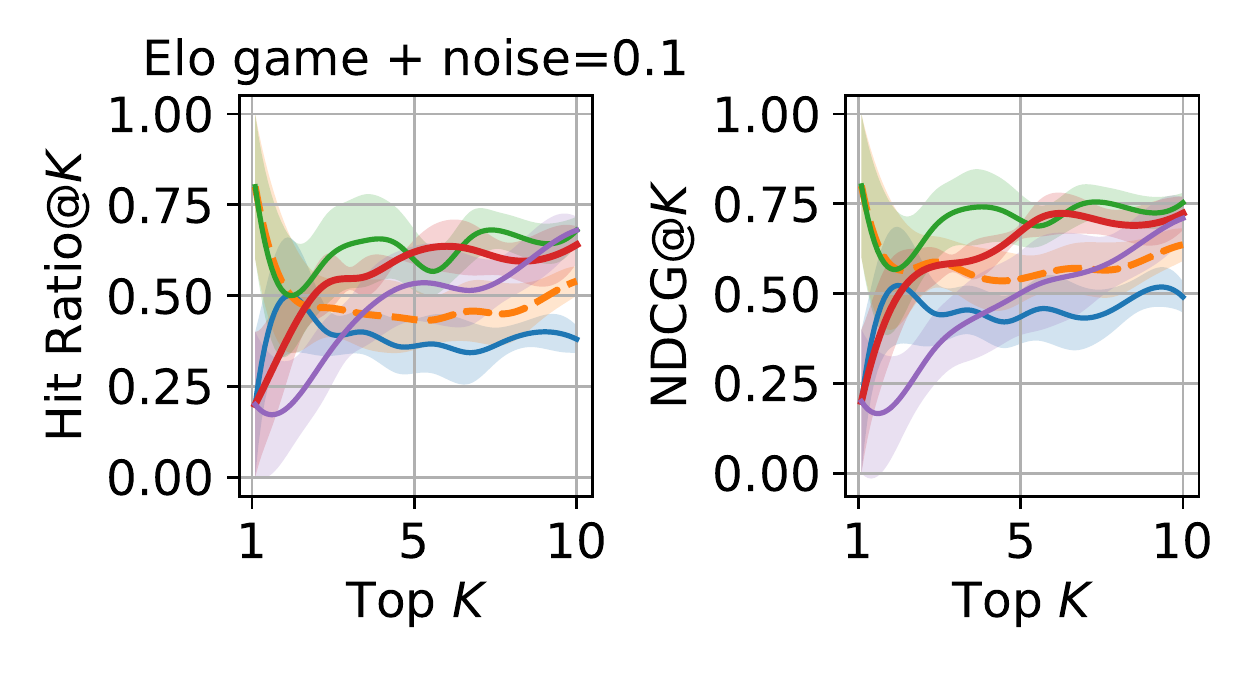}
    \end{subfigure}
    \caption{Additional results of different $\gamma$ of our $\ourmethod$ under top-$k$ performance. 
    The best top-1 curves are in dashed lines. The best top-$k$ curves are in bold.
    }
    \label{fig:elo-topk}
\end{figure}

\begin{figure}[th!]
    \centering
    \begin{subfigure}
        \centering
        \includegraphics[width=0.6\linewidth]{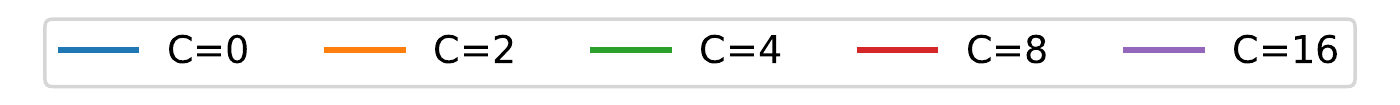}
    \end{subfigure}
    \begin{subfigure}
    \centering
        \includegraphics[width=0.48\linewidth]{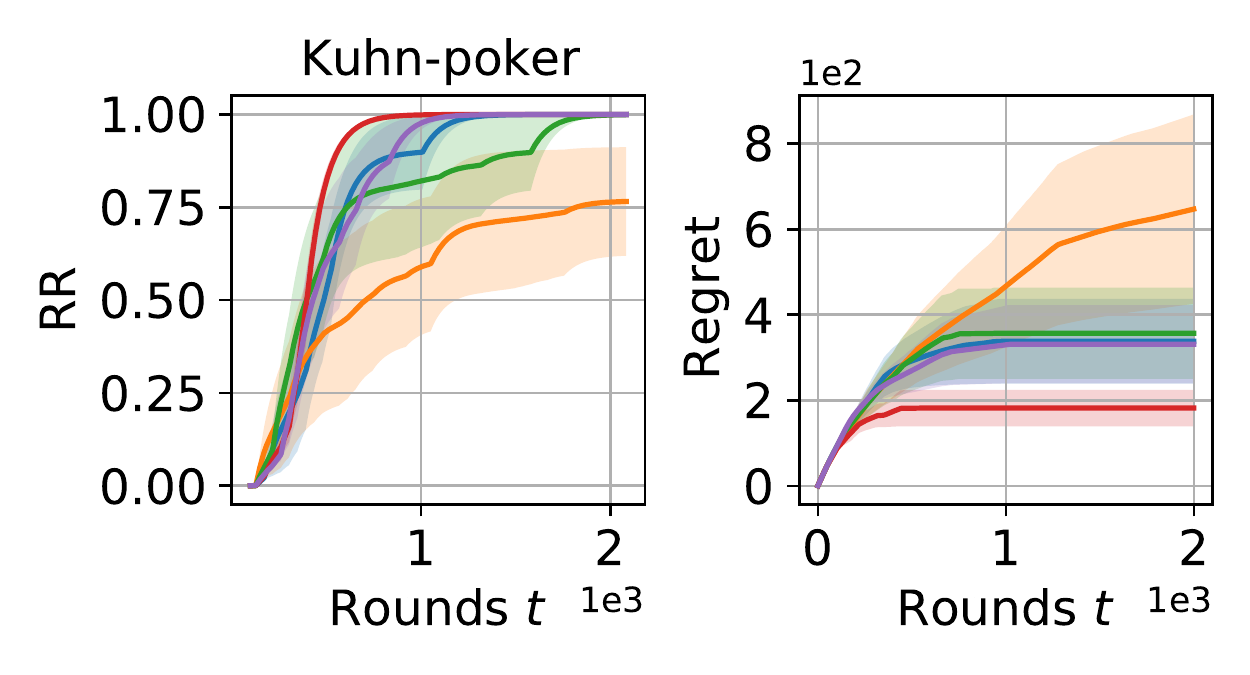}
    \end{subfigure}
    \begin{subfigure}
    \centering
        \includegraphics[width=0.48\linewidth]{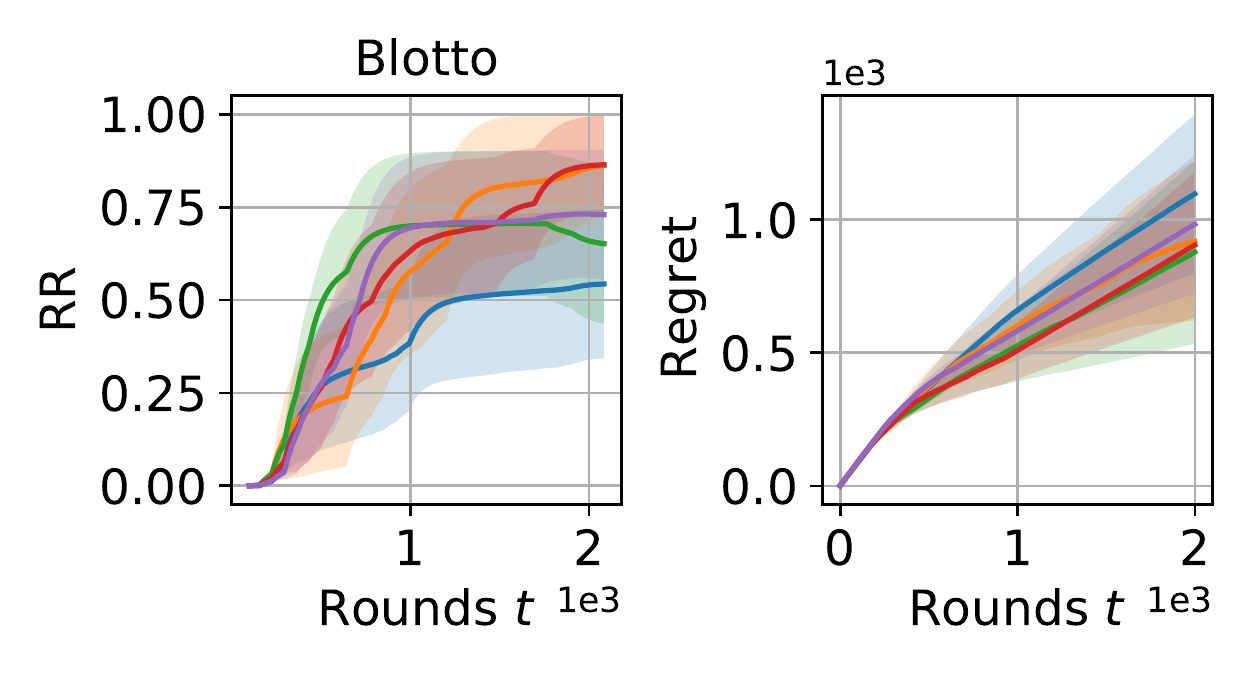}
    \end{subfigure}
    \caption{Results of different parameters of $C$ of $\ourmethodii$. 
    }
    \label{fig:cdim-melo-real-game}
\end{figure}

\begin{figure}[th!]
    \centering
    \begin{subfigure}
        \centering
        \includegraphics[width=0.8\linewidth]{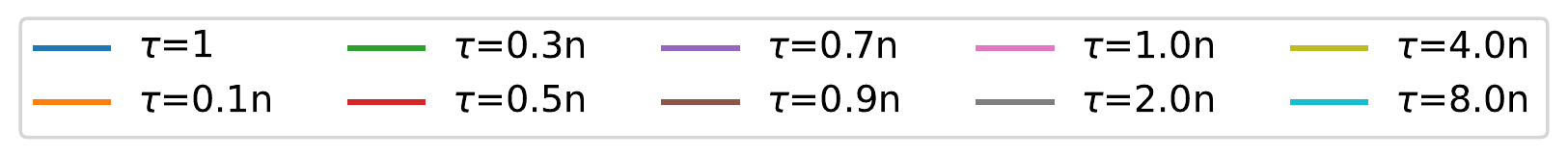}
    \end{subfigure}\\
    \begin{subfigure}
    \centering
        \includegraphics[width=0.48\linewidth]{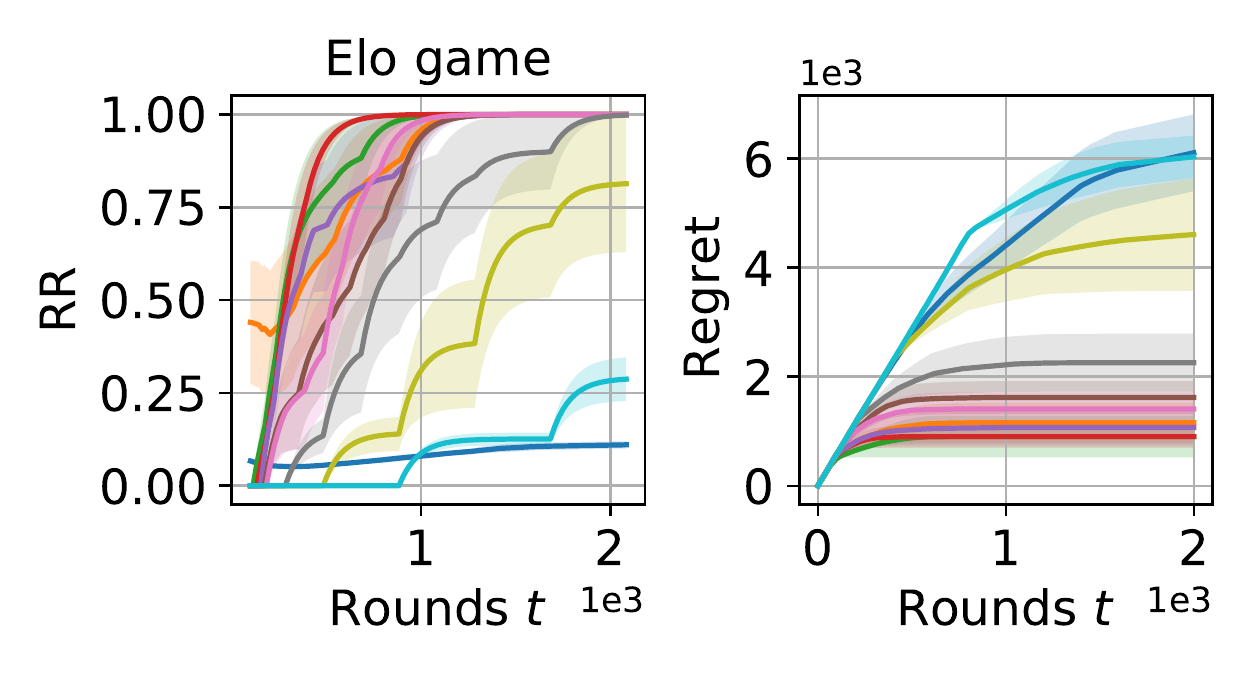}
    \end{subfigure}
    \begin{subfigure}
    \centering
        \includegraphics[width=0.48\linewidth]{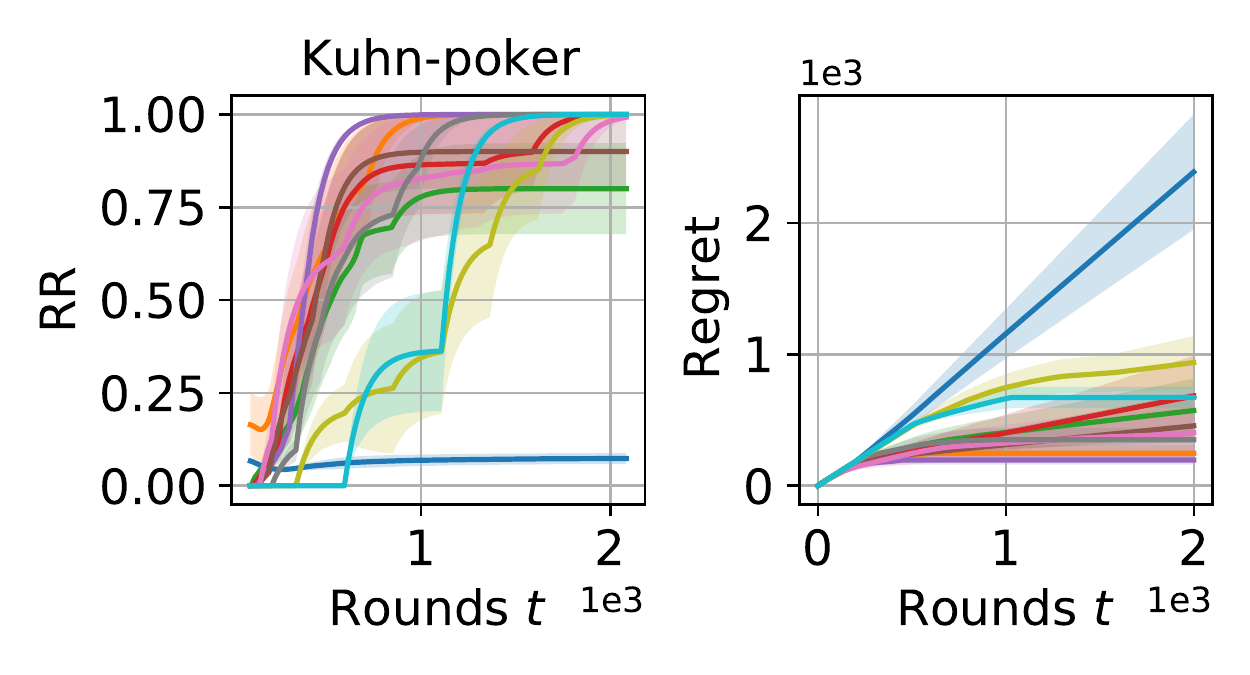}
    \end{subfigure}
    \caption{Results of different batch size $\tau$ of  $\ourmethod$ on Elo game and $\ourmethodii$ on Kuhn-poker.  
    }
    \label{fig:batch-ablation-study}
\end{figure}

\section{Additional Experiments}\label{app:experiments}



Table \ref{tab:real-game-statistics} gives the details of the games used in experiments. The sizes of the top SSCC of six transitive games are equal to 1, which indicates that there exists a top player and its winning probabilities against other players are greater than $0.5$. For the intransitive games,  the top SSCC size of AlphaStar is equal to 1, but there also exist cycles among players outside of the top SSCC. The top SSCC is the smallest set of players such that the winning probabilities of outside players against inside players are less than 0.5. 

\fig \ref{fig:elo-topk} shows the comparison between different $\gamma$ of $\ourmethod$ under top-$k$ identification. $\gamma$ trade off exploration and exploitation in selecting arm-pairs. Smaller $\gamma$ leads to tighter candidate sets then leads to better top-1 performance, but too small $\gamma$ may result in that the top player is excluded from the candidate set. 
On Elo game, its three variants and Triangular game, $\ourmethod$ has the best performance when $\gamma=0.6$. The performance of top-1 gradually drops with $\gamma$ increasing. But it misidentifies the best player when $\gamma=0.4$. 
By contrast, larger $\gamma$ leads to larger candidate set then lead to better top-$k$ performance, but too large $\gamma$ results in much useless exploration on top players identification. In general, larger $\gamma$ has better top-$k$ performance, but $\gamma=1.2$ is worse than $\gamma=1$ because of too large candidate set and too much exploration. On Elo game, $\ourmethod$ has best top-$1$ performance on $\gamma=0.6$, best top-$k$ performance on $\gamma=1$, and $\gamma=0.8$ balances the top-1 and top-$k$ identification. On Transitive game, $\ourmethod$ has best top-$1$ and top-$k$ performance when $\gamma=0.4$. 

\fig \ref{fig:cdim-melo-real-game} gives results $\ourmethodii$ with different dimensions of $c$.
Overall, $C=8$ returns best performance. The performance of top-1 identification and cumulative of $\ourmethodii$ gradually increases as the dimension $C$ increases, but the performance drops when $C=16$. 
In general, larger $C$ leads to better approximation of the intransitive skills, but this also requires more samples for training. 

{\fig \ref{fig:batch-ablation-study} shows results of $\ourmethod$ on Elo game and of $\ourmethodii$ on Kuhn-poker with different batch size $\tau$. When the batch size $\tau=1$, both $\ourmethod$ and $\ourmethodii$ have bad performance on top-1 player identification and large cumulative reward, which may because of that the objective function is not strongly convex and hard to convergence under this setting. On the other hand, when selecting too large $\tau$, such as $4.0n,\ 8.0n$, the number of updates is less than that of small $\tau$ with the same sampling rounds. The $\ourmethod$ on the transitive Elo game gets satisfactory top-1 identification and regret performance at around $\tau=0.5n~1.0n$, and the $\ourmethod$ on the intransitive Kuhn-poker obtains best top identification and regret performance at around $\tau=0.7n$.}

\newpage
\section{Theories and Proofs}\label{app:proofs}
We provide detailed proofs below.
\subsection{Proofs of Lemmas}
\begin{myprops}[\cite{saha2020regret} ]\label{"props1"} 
We select a batch of feature vectors ${X_1, X_2, \dots, X_\tau}$ from a distribution $v$ and $\|X\|\leq 2$. Let $V_{\tau+1}=\sum_{t=1}^\tau (X_tX_t^\top)$, $\Sigma=\mathbb{E}_{X\overset{iid}{\sim}v}[XX^\top]$.  There exist two positive constant $C_1,C_2$ such that if the batch size $\tau$ is appropriately selected as $\tau \geq 2\left(\frac{C_1\sqrt{n}+C_2\sqrt{\log(1/\delta)}}{\lambda_{\min}(\Sigma)}\right)^2+\frac{4C}{\lambda_{\min}(\Sigma)}$, then $\lambda_{\min}(V_{\tau+1})\geq C $ holds with probability at least $1-\frac{1}{\delta}$.
\end{myprops}
The Proposition \ref{"props1"} provides guidance about the selection of batch size $\tau$ and is a corollary of Lemma 12 in \citep{saha2020regret}. 
In \alg \ref{alg:dueling-sgd-UCB}, the first batch of pairs are randomly selected for initializing the maximum likelihood estimator, the rest batches of pairs are selected with maximum uncertainty based on UCB estimation. 
The batch size $\tau$ can be suitably to achieve low MLE error and the $\alpha$-strongly convex condition. 

\lemmai*
\begin{proof}[{Proof of Lemma \ref{"lm:mylemma1"}}]
Our setting is a special case of generalized linear bandits(GLM)\cite{li2017provably}, and we can directly get our {Lemma } \ref{"lm:mylemma1"} by substitute $e_x-e_y$ into {Lemma 2} of \cite{li2017provably}.

According to \cite{li2017provably}, we have
\begin{equation}
\sum_{i=\tau+1}^{\tau+t}\left\|e_{x_i}-e_{y_i}\right\|_{V_{i}^{-1}}^{2} \leq 2 \log \frac{\operatorname{det} V_{t+1}}{\operatorname{det} V_{\tau+1}} \leq 2 d \log \left(\frac{\operatorname{tr}\left(V_{\tau+1}\right)+t}{d}\right)-2 \log \operatorname{det} V_{\tau+1}.
\end{equation}

The trace of $V_{\tau+1}$ is $\text{tr}(V_{\tau+1})=\sum_{i=1}^ \tau \text{tr}\left((e_{x_i}-e_{y_i})(e_{x_i} - e_{y_i})^\top\right)=\sum_{i=1}^\tau\left\|e_{x_i}-e_{y_i}\right\|^2=2\tau$, and the determinant satisfies $det(V_{\tau+1})=\prod_{i=1}^{d} \lambda_{i} \geq \lambda_{\min }^{d}\left(V_{\tau+1}\right) \geq 1$ for we assume $\lambda_{\min}(V_{\tau+1})\geq 1$, where $\{\lambda_i\}$ are the eigenvalues of $V_{\tau+1}$.
Applying the Cauchy-Schwarz inequality, we have 

\begin{equation}
\sum_{i=\tau+1}^{\tau+t}\left\|(e_{x_i}-e_{y_i})\right\|_{V_{i}^{-1}} \leq \sqrt{t \sum_{i=\tau+1}^{t}\left\|(e_{x_i}-e_{y_i})\right\|_{V_{i}^{-1}}^{2}} \leq \sqrt{2 t d \log \left(\frac{t+2\tau}{d}\right)}
\end{equation}
Thus the proof of {Lemma } \ref{"lm:mylemma1"} is complete.

\end{proof}

\lemmaii*
\begin{proof}[{Proof of Lemma \ref{"lm:mylemma2"}} ]
\label{'proof lemma2'}
Relying on the Lemma 1 in \cite{ding2021efficient} and Proposition \ref{"props1"}, we set the batch size $\tau\geq \tau_1$ in \eq \eqref{"eq:tau"}, thus we have $\|\hat{r}_{j\tau}-r^*\|\leq 1$ with probability at least $1-\frac{1}{T^2}$ for any $j\geq 1$. So we have $\|\hat{r}_{j\tau}-\hat{r}_{\tau}\|\leq 2$ with probability at least $1-\frac{2}{T^2}$. In \alg \ref{alg:dueling-sgd-UCB}, we project the SGD estimation $\tilde{r}_j$ into the convex hall $\mathcal{C}=\{r:\|r-\hat{r}_\tau\|\leq2\}$, thus we have $\tilde{r}_j \in \{r:\|r-{r}^*\|\leq3\}$ with probability at least $1-\frac{1}{T^2}$. And from Proposition 1, if the batch size $\tau\geq \tau_2$ defined in \eq \eqref{"eq:tau"}, then $\lambda_{\min}(\sum_{t=(j-1)\tau+1}^{j\tau}(e_{x_t}-e_{y_t})(e_{x_t}-e_{y_t})^\top)\geq\frac{\alpha}{c_3}$ holds with probability at least $1-\frac{1}{T^2}$. Thus with probability at least $1-\frac{1}{T^2}$, the objective function $l_{j,\tau}(r)$ in \eq \eqref{eq:ljtau_mle_obj} is an $\alpha$-strongly convex function when $r\in \{r : \|r-r^*\|\leq 3\}$. 

Relying on the above analysis and according to Lemma 2 in \cite{ding2021efficient}, we have that if the batch size $\tau$ is chosen as \eq \eqref{"eq:tau"}, then the convergence rate of the $\alpha$-strongly convex objective function $l_{j,\tau}(r)$ satisfies that  \begin{equation}
\label{eq:bar-hat}
\left\|\bar{r}_{j}-\hat{r}_{j \tau}\right\| \leq \frac{\tau}{\alpha} \sqrt{\frac{1+\log j}{j}}
\end{equation}
for all $j\geq1$ with probability at least $1-\frac{3}{T^2}$.

From Proposition \ref{"props1"} and $\tau\geq \tau_1$ in \eq \eqref{"eq:tau"}, we have $\lambda_{min}(V_{\tau+1})\geq \frac{4(n+2\log T)}{c_1^2} \geq 1$ with probability at least $1-\frac{1}{T^2}$ because $c_1$ is upper bounded by $\frac{1}{4}$. Then according to the Lemma 3 in \cite{li2017provably}, we have that  the event
\begin{equation}
\label{eq:hat-star}
\mathcal{E}=\{\|\hat{r}_t-r^*\|_{V_{t}}\leq\frac{1}{2c_{1}} \sqrt{\frac{n}{2} \log \left(1+\frac{2 t}{n}\right)+2 \log T},t \geq \tau\}
\end{equation}
holds with probability at least $1-\frac{2}{T^2}$.
Combining the \eq \eqref{eq:bar-hat} and \eq \eqref{eq:hat-star} and applying the union bound, we can get that for all $t\geq\tau$ the event $E_1(t)$ holds with probability at least $1-\frac{5}{T^2}$.
\begin{equation}
\begin{aligned}
|(e_x-e_y)^T(\bar{r}_j-r^*)|\leq&|(e_x-e_y)^T(\bar{r}_j-\hat{r}_{j\tau})|+|(e_x-e_y)^T(\hat{r}_{j\tau}-r^*)|\\
\stackrel{(1)}{\leq}&\sqrt{2}\frac{\tau}{\alpha}\sqrt{\frac{1+\log j}{j}}+|(e_x-e_y)^T(\hat{r}_{j\tau}-r^*)|\\
=&g_2(j)\frac{\sqrt{2}}{\sqrt{j}}+|(e_x-e_y)^T(\hat{r}_{j\tau}-r^*)|\\
\stackrel{(2)}{\leq}&g_2(j)\frac{\sqrt{2}}{\sqrt{j}}+\|\hat{r}_{j\tau}-r^*\|_{V_{j\tau+1}}\|e_x-e_y\|_{V_{j\tau+1}^{-1}}\\
\stackrel{(3)}{\leq}& g_2(j)\frac{\sqrt{2}}{\sqrt{j}}+ g_1(j\tau)\|e_x-e_y\|_{V_{j\tau+1}^{-1}}
\end{aligned}
\end{equation}
Inequality (1) holds because the \eq \eqref{eq:bar-hat}; inequality (2) is due to the Cauchy-Schwartz inequality, where $V_{j\tau+1}$ is a positive definite matrix with probability at least $1-\frac{1}{T^2}$. (3) holds from the \eq \eqref{eq:hat-star}.  
\end{proof}

\lemmaiii*
\begin{proof}[{Proof of Lemma \ref{"mylemma3"}}]
Following the proof of Lemma \ref{"lm:mylemma2"}, for any $j\geq1$ we have that:
\begin{equation}
\label{"lm3 proof 1"}
\begin{aligned}
|(e_x-e_y)^T(\bar{r}_j-r^*)|\leq&|(e_x-e_y)^T(\bar{r}_j-\hat{r}_{j\tau})|+|(e_x-e_y)^T(\hat{r}_{j\tau}-r^*)|+|(e_x-e_y)^T(r^{*}-\hat{r}_{t-1})|+|(e_x-e_y)^T(\hat{r}_{t-1}-r^*)|\\
{\leq}&\sqrt{2}\frac{\tau}{\alpha}\sqrt{\frac{1+\log j}{j}}+|(e_x-e_y)^T(\hat{r}_{j\tau}-{r}^*)|+2|(e_x-e_y)^T(\hat{r}_{t-1}-r^*)|\\
\stackrel{(1)}{\leq}&g_2(j)\frac{\sqrt{2}}{\sqrt{j}}+\|\hat{r}_{j\tau}-r^*\|_{V_{j\tau+1}}\|e_x-e_y\|_{V_{j\tau+1}^{-1}}+\|\hat{r}_{t}-r^*\|_{V_t}\|e_x-e_y\|_{V_{t}^{-1}}\\
\stackrel{(2)}{\leq}& g_2(j)\frac{\sqrt{2}}{\sqrt{j}}+ g_1(j\tau)\|e_x-e_y\|_{V_{j\tau+1}^{-1}}+2g_1(t)\|e_x-e_y\|_{V_{t}^{-1}}\\
\stackrel{(3)}{\leq}& g_2(j)\frac{\sqrt{2}}{\sqrt{j}}+ g_1(j\tau)\sqrt{2nj\tau\log \left(\frac{(j+1)\tau+1}{n}\right)}+2g_1(t)\|e_x-e_y\|_{V_{t}^{-1}}\\
{\leq}& g_2(j)\frac{\sqrt{2}}{\sqrt{j}}+ g_1(T)\sqrt{2nT\log \left(\frac{T+\tau}{n}\right)}+2g_1(t)\|e_x-e_y\|_{V_{t}^{-1}}
\end{aligned}
\end{equation}
Inequality (1) is due to the Cauchy-Schwartz inequality, where $V_{j\tau+1}$ and $V_{t}$ is a positive definite matrix with probability at least $1-\frac{1}{T^2}$. (2) holds because the \eq \eqref{eq:hat-star}. (3) holds from the Lemma \ref{"lm:mylemma1"}. 
Define the constant 

$C=\sqrt{2nT\log \left(\frac{T+\tau}{n}\right)}$, the inequality
$\left|(e_x-e_y)^T\left(\bar{r}_{j}-r^{*}\right)\right| \leq g_2(j)\frac{\sqrt{2}}{\sqrt{j}}+ g_1(T)\sqrt{2nT\log \left(\frac{T+\tau}{n}\right)}+2g_1(t)\|e_x-e_y\|_{V_{t}^{-1}}$ 
holds with high probability as above derivation.
This is equivalent to the following:
\begin{equation}\label{eq:lemma3proof}
\begin{aligned}
&(e_y-e_{x^*})(\bar{ r}_j- r^*)\leq g_2(j)\frac{\sqrt{2}}{\sqrt{j}}+ g_1(T)C+2g_1(t)\|e_x-e_y\|_{V_{t}^{-1}}\\
\Rightarrow & 
(e_{x^*}-e_y)( \bar{ r}_j- r^*)+2g_{1}(t)\|e_{x^*}-e_y\|_{V_t^{-1}}+g_{2}(j) \sqrt{\frac{2}{j}}+  g_1(T)C\geq 0\\
\Rightarrow & 
 \bar{ r}_j(x^*)- \bar{ r}_j(y)- r^*(x^*)+ r^*(y)+2g_{1}(t)\|e_{x^*}-e_y\|_{V_t^{-1}}+g_{2}(j) \sqrt{\frac{2}{j}}+ g_1(T)C\geq 0\\
 \Rightarrow & 
 \bar{ r}_j(x^*)- \bar{ r}_j(y)+g_{1}(t)\|e_{x^*}-e_y\|_{V_t^{-1}}\geq  r^*(x^*)-  r^*(y)-g_{2}(j) \sqrt{\frac{2}{j}}-g_1(T)C\\
 \Rightarrow & 
 \bar{ r}_j(x^*)- \bar{ r}_j(y)+g_{1}(t)\|e_{x^*}-e_y\|_{V_t^{-1}}\geq  r^*(x^*)- r^*(y)-\frac{\tau}{\alpha}\sqrt{1+\log j} \sqrt{\frac{2}{j}}-g_1(T)C\\
\stackrel{(1)}{\Rightarrow}&
h(x^*,y)\geq  r^*(x^*)- r^*(y)-\frac{\tau}{\alpha}\sqrt{1+\log j} \sqrt{\frac{2}{j}}-g_1(j)C\\
\stackrel{(2)}{\Rightarrow}&
h(x^*,y)\geq \Delta-\frac{\tau}{\alpha}\sqrt{1+\log j} \sqrt{\frac{2}{j}}-g_1(j)C\stackrel{(3)}{\geq} 0\\
\stackrel{}{\Rightarrow}& h(x^*,y)\geq  0.
\end{aligned}
\end{equation}
Inequality (1) holds for $\gamma=2g_1(t)$ and the definition of $h(x,y)$ in \eq \eqref{"eq:UCB_esti"}; (2) follows from $\Delta$ is the difference between optimal and sub-optimal arms, so $\forall y\in[n]/{x^*},\ \Delta\leq   r^*(x^*)-  r^*(y)$. (3) holds when $\alpha\geq \frac{\sqrt{2}\tau\sqrt{1+\log j}}{(\Delta-g_1(j)C)\sqrt{j}}$ and $\Delta>g_1(j)C$.
\end{proof}

\subsection{Proofs of Theorem \ref{"thm:Elo-SGD-UCB"}}
\begin{proof}[{Proof of theorem \ref{"thm:Elo-SGD-UCB"}}]
We use $b_t$ to denote the regret at round $t$ and recall that
$$b_t =   r^*_{x^*}-\frac{1}{2}( r^*_{x_t}+ r^*_{y_t})=\frac{(e_{x^*}-e_{x_t})^T r^*+(e_{x^*}-e_{y_t})^T r^*}{2},$$ 
and the index of batch is $j=\lfloor(t-1)/\tau \rfloor$.
Then we have that when $t>\tau$:
\begin{equation}\label{"eq:regret"}
\begin{aligned}
2b_t &=\left(e_{x^*}-e_{x_t}\right)^{T} {r}^{*}+\left(e_{x^*}-e_{y_t}\right)^{T}  r^{*} \\
&=\left(e_{x^*}-e_{x_t}\right)^{T}  \bar{ r}_{j}+\left(e_{x^*}-e_{x_t}\right)^{T}\left( r^*-\bar{ r}_j \right) \\
& \quad \quad + \left(e_{x^*}-e_{y_t}\right)^{T}  \bar{ r}_{j}+\left(e_{x^*}-e_{y_t}\right)^{T}\left( r^*-\bar{  r}_j \right) \\
& \stackrel{(1)}{\leq} \gamma\left\|e_{x^*}-e_{x_t}\right\|_{V_{t}^{-1}}+ \left(e_{x^*}-e_{x_t}\right)^{T}\left(  r^*-\bar{  r}_j \right)
\\
& \quad \quad +\gamma\left\|e_{x^*}-e_{y_t}\right\|_{V_{t}^{-1}}
+\left(e_{x^*}-e_{y_t}\right)^{T}\left(  r^*-\bar{  r}_j \right) \\
& \stackrel{(2)}{\leq} \gamma\left\|e_{x^*}-e_{x_t}\right\|_{V_{t}^{-1}}
+ g_{1}(j\tau)\|e_{x^*}-e_{x_t}\|_{V_{j\tau+1}^{-1}}+g_{2}(j) \frac{2}{\sqrt{j}}\\
&\quad \quad +\gamma\left\|e_{x^*}-e_{y_t}\right\|_{V_{t}^{-1}}
+g_{1}(j\tau)\|e_{x^*}-e_{y_t}\|_{V_{j\tau+1}^{-1}}+g_{2}(j) \frac{2}{\sqrt{j}} \\
& \stackrel{(3)}{\leq} \gamma\left\|e_{x_t}-e_{y_t}\right\|_{V_{t}^{-1}}
+ g_{1}(j\tau)\|e_{x_t}-e_{y_t}\|_{V_{j\tau+1}^{-1}}+g_{2}(j) \frac{2}{\sqrt{j}}\\
&\quad \quad +\gamma\left\|e_{x_t}-e_{y_t}\right\|_{V_{t}^{-1}}
+g_{1}(j\tau)\|e_{x_t}-e_{y_t}\|_{V_{j\tau+1}^{-1}}+g_{2}(j) \frac{2}{\sqrt{j}} 
\end{aligned}
\end{equation}
Inequality (1) holds for $x_t,y_t$ belong to candidate set $\mathcal{S}$, inequality (2) holds with probability at least $1-\frac{5}{T^2}$ following from {Lemma} \ref{"lm:mylemma2"}. 
{Inequality (3) holds according to {Lemma} \ref{"mylemma3"} where 
we have $x^*\in\mathcal{S}$ with probability at least $1-\frac{5}{T^2}$ and our arm-pair selection strategy.
}

By above analysis, we have \eq \eqref{"eq:regret"} holds with probability at least $1-\frac{10}{T^2}$ at each round $t$.
Next, we give the results on cumulative regret bound:
\begin{equation}\label{"bt"}
\begin{aligned}
R(T)&=\sum_{t=1}^{\tau}b_t+\sum_{t=\tau+1}^T b_t\\
&\stackrel{(1)}{\leq}\tau*\Delta_{\max}+\sum_{t=\tau+1}^T b_t\\
&\stackrel{(2)}{\leq}\tau*\Delta_{\max}+\sum_{t=\tau+1}^T\frac{1}{2}\big [ \gamma\left\|e_{x_t}-e_{y_t}\right\|_{V_{t}^{-1}}
+ g_{1}(j\tau)\|e_{x_t}-e_{y_t}\|_{V_{j\tau+1}^{-1}}+g_{2}(j) \frac{2}{\sqrt{j}}\\
&\quad+\gamma\left\|e_{x_t}-e_{y_t}\right\|_{V_{t}^{-1}}
+g_{1}(j\tau)\|e_{x_t}-e_{y_t}\|_{V_{j\tau+1}^{-1}}+g_{2}(j) \frac{2}{\sqrt{j}} \big] \\
&\stackrel{(3)}{\leq} \tau*\Delta_{\max}+\sum_{t=\tau+1}^T \gamma \left\|e_{x_t}-e_{y_t}\right\|_{V_{t}^{-1}}+\tau\sum_{t=\tau+1}^T g_1(t-1)\|e_{x_t}-e_{y_t}\|_{V_t^{-1}}+2g_2(j)\sum_{t=\tau+1}^T\frac{1}{\sqrt{j}}\\
&\stackrel{(4)}{\leq}\tau*\Delta_{\max}+2g_1(T) \sqrt{2nT\log(\frac{\tau+T}{n})}+\tau g_1(T){\sqrt{2nT\log\left(\frac{\tau+T}{n}\right)}}+2g_2(J)\sum_{t=\tau+1}^T\frac{1}{\sqrt{j}}\\
&\stackrel{(5)}{\leq}\tau*\Delta_{\max}+(2+\tau)g_1(T) \sqrt{2nT\log(\frac{\tau+T}{n})}+4g_2(J)\sqrt{\tau T}
\end{aligned}
\end{equation}
\end{proof}

Inequality (1) holds because $\Delta_{\max}=\max_i r^*_i-\min_i r^*_i$. As the \eq \eqref{"bt"}, (2) holds with probability at least $1-\frac{10}{T}$ by the union bound. Consider that there are $\tau$ exploration tracks with length $T$ and we only use the nodes at $j\tau$, thus (3) holds due to $\sum_{t=\tau+1}^T g_1(j\tau)\|e_{x_t}-e_{y_t}\|_{V_{j\tau+1}^{-1}} \leq \sum_{t=\tau+1}^T g_1(t-1)\|e_{x_t}-e_{y_t}\|_{V_{t}^{-1}}$, where $j=\lfloor\frac{t-1}{\tau}\rfloor$. As the analysis in the proof of Lemma \ref{"lm:mylemma2"}, $\lambda_{min}(V_{\tau+1}) \geq 1$ holds with probability at least $1-\frac{1}{T^2}$, then we have $\sum_{i=\tau+1}^{t}\left\|\left(e_{x_i}-e_{y_i}\right)\right\|_{V_{i}^{-1}} \leq \sqrt{2 n t \log \left(\frac{2\tau+t}{n}\right)}$ by the Lemma \ref{"lm:mylemma1"}. Also because the {\ourmethod} sets $\gamma=2g_1(t)$, $J=\lfloor\frac{T}{\tau}\rfloor$, thus (4) holds. 
(5) holds because $\sum_{t=\tau+1}^T\frac{1}{\sqrt{j}}\leq 2\sqrt{\tau T}$ {see \eq (21) in the Appendix of \cite{ding2021efficient} }. 

Thus the proof of Theorem \ref{"thm:Elo-SGD-UCB"} is completed.



\end{document}